\documentclass{article}

\usepackage[final]{neurips_2020}

\newcommand{\brac}[1]{\left( #1 \right) }
\newcommand{\ml}[1]{\mathcal{ #1 } }

\newcommand{\R}{\mathbb{R}}

\usepackage{graphicx}
\usepackage{amsmath, amssymb}
\usepackage{enumitem, array}
\usepackage[flushleft]{threeparttable}
\usepackage{algpseudocode}

\newcommand{\Req}{\textbf{Require:}\hspace*{0.5em}}

\newcommand{\X}{\hspace*{3mm}}
\newcommand{\XX}{\X\X}
\newcommand{\XXX}{\X\X\X}
\newcommand{\XXXX}{\X\X\X\X}
\newcommand{\XXXXX}{\X\X\X\X\X}

\newcommand{\cm}[1]{$\triangleright$ #1}

\usepackage{hyperref}
\hypersetup{hidelinks,
backref=true,
pagebackref=true,
hyperindex=true,
breaklinks=true,
colorlinks=true,
urlcolor=blue,
bookmarks=true,
bookmarksopen=false,
pdftitle={Title},
pdfauthor={Author}}
\usepackage[utf8]{inputenc} 
\usepackage[linesnumbered,ruled]{algorithm2e}
\usepackage[T1]{fontenc}    
\usepackage{url}            
\usepackage{booktabs}       
\usepackage{amsfonts,dsfont}       
\usepackage{nicefrac}       
\usepackage{microtype}      
\usepackage{bbding}
\usepackage{pifont}
\usepackage{wasysym}
\usepackage{amssymb}
\newcommand{\cmark}{\ding{51}}%
\newcommand{\xmark}{\ding{55}}%
\usepackage{amsthm}

\newtheorem{theorem}{Theorem}

\newtheorem{lemma}[theorem]{Lemma}
\newtheorem{proposition}[theorem]{Proposition}
\usepackage{amsthm}
\usepackage{microtype}
\usepackage{graphicx}
\usepackage{subfigure}
\usepackage{booktabs} 
\usepackage{amsmath}
\usepackage{multirow}
\usepackage{xcolor}
\usepackage{booktabs}

\newcommand{\vy}{\mathbf{y}}
\newcommand{\vytrue}{\vy^*}

\newcommand{\vx}{\mathbf{x}}
\newcommand{\param}{\Theta}

\newcommand{\vp}{\mathbf{p}}
\newcommand{\vg}{\mathbf{g}}
\newcommand{\vq}{\mathbf{t}}
\newcommand{\vv}{\mathbf v}
\newcommand{\vz}{\mathbf z}
\newcommand{\pmlabels}{\mathbf{\varepsilon}}
\newcommand{\pmlabelstrue}{\pmlabels^\ast}
\newcommand{\noise}{\Delta} 
\newcommand{\func}{\ml{N}}
\newcommand{\smax}{\ml{S}}
\newcommand{\coeff}{\kappa}
\newcommand{\correct}{C}
\newcommand{\wrong}{W}
\newcommand{\E}{\mathbb E}
\newcommand{\1}{\mathds{1}}
\newcommand{\PP}{\mathbb{P}}

\title{Early-Learning Regularization Prevents Memorization of Noisy Labels}

\author{%
  Sheng Liu \\
  Center for Data Science\\
  New York University\\
  \texttt{shengliu@nyu.edu} \\
  \And
  Jonathan Niles-Weed\\
   Center for Data Science, and\\
   Courant Inst. of Mathematical Sciences\\
   New York University\\
   \texttt{jnw@cims.nyu.edu} \\
  \And
  Narges Razavian\\
  Department of Population Health, and\\
  Department of Radiology\\
  NYU School of Medicine\\
   \texttt{narges.razavian@nyulangone.org} \\
  \And
  Carlos Fernandez-Granda\\
  Center for Data Science, and\\
   Courant Inst. of Mathematical Sciences\\
   New York University\\
   \texttt{cfgranda@cims.nyu.edu} \\
}

\begin{document}

\maketitle

\begin{abstract}
  We propose a novel framework to perform classification via deep learning in the presence of noisy annotations. When trained on noisy labels, deep neural networks have been observed to first fit the training data with clean labels during an ``early learning'' phase, before eventually memorizing the examples with false labels. We prove that early learning and memorization are fundamental phenomena in high-dimensional classification tasks, even in simple linear models, and give a theoretical explanation in this setting. Motivated by these findings, we develop a new technique for noisy classification tasks, which exploits the progress of the early learning phase. In contrast with existing approaches, which use the model output during early learning to detect the examples with clean labels, and either ignore or attempt to correct the false labels, we take a different route and instead capitalize on early learning via regularization. There are two key elements to our approach. First, we leverage semi-supervised learning techniques to produce target probabilities based on the model outputs. Second, we design a regularization term that steers the model towards these targets, implicitly preventing memorization of the false labels. The resulting framework is shown to provide robustness to noisy annotations on several standard benchmarks and real-world datasets, where it achieves results comparable to the state of the art.
 \end{abstract}

\section{Introduction}

Deep neural networks have become an essential tool for classification tasks~\citep{krizhevsky2012imagenet,he2016deep,girshick2014rich}. These models tend to be trained on large curated datasets such as CIFAR-10~\citep{krizhevsky2009learning} or ImageNet~\citep{deng2009imagenet}, where the vast majority of labels have been manually verified. Unfortunately, in many applications such datasets are not available, due to the cost or difficulty of manual labeling (e.g.~\citep{guan2018said,pechenizkiy2006class,liu20a,ait2010high}). However, datasets with lower quality annotations, obtained for instance from online queries~\citep{blum2003noise} or crowdsourcing~\citep{yan2014learning,yu2018learning}, may be available. Such annotations inevitably contain numerous mistakes or \emph{label noise}. It is therefore of great importance to develop methodology that is robust to the presence of noisy annotations.

 When trained on noisy labels, deep neural networks have been observed to first fit the training data with clean labels during an \emph{early learning} phase, before eventually \emph{memorizing} the examples with false labels~\citep{arpit2017closer,zhang2016understanding}. In this work we study this phenomenon and introduce a novel framework that exploits it to achieve robustness to noisy labels. Our main contributions are the following:
 
\begin{itemize}[leftmargin=*]
\item In Section~\ref{sec:linear} we establish that early learning and memorization are fundamental phenomena in high dimensions, proving that they occur even for simple linear generative models.
\item In Section~\ref{sec:methodology} we propose a technique that utilizes the early-learning phenomenon to counteract the influence of the noisy labels on the gradient of the cross entropy loss. This is achieved through a regularization term that incorporates target probabilities estimated from the model outputs using several semi-supervised learning techniques.
\item In Section~\ref{sec:results} we show that the proposed methodology achieves results comparable to the state of the art on several standard benchmarks and real-world datasets. We also perform a systematic ablation study to evaluate the different alternatives to compute the target probabilities, and the effect of incorporating mixup data augmentation~\citep{zhang2017mixup}.
\end{itemize}

\begin{figure}[t]
    \begin{tabular}{>{\centering\arraybackslash}m{0.13\linewidth} >{\centering\arraybackslash}m{0.4\linewidth} >{\centering\arraybackslash}m{0.4\linewidth}}
    & {\small Clean labels} & {\small Wrong labels}\\
    {\small Cross Entropy}
    & \includegraphics[width=\linewidth]{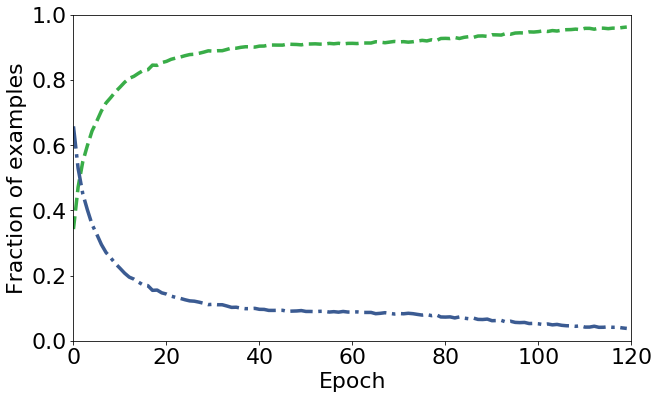}&
    \includegraphics[width=\linewidth]{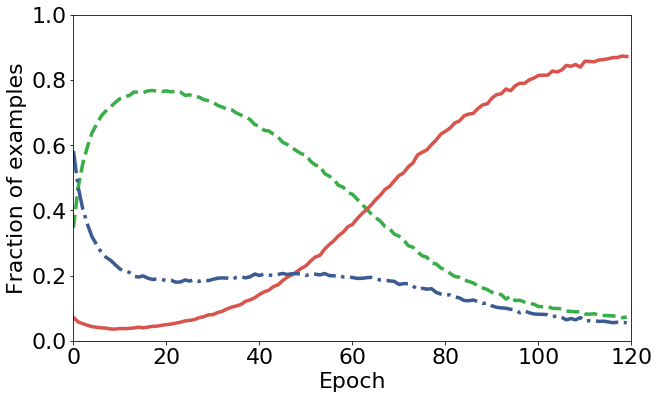}\\
    {\small \shortstack{Early-learning\\Regularization}}
    & \includegraphics[width=\linewidth]{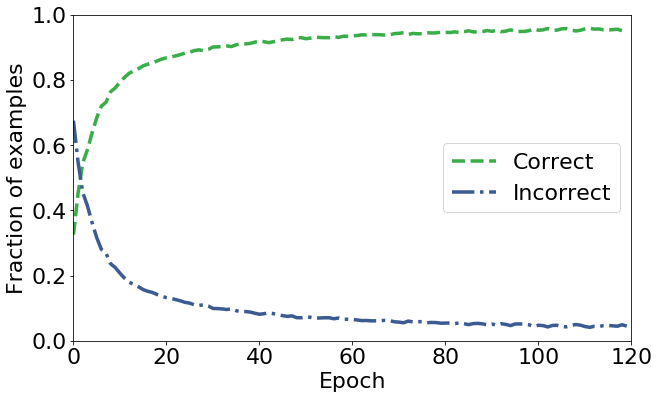}&
    \includegraphics[width=\linewidth]{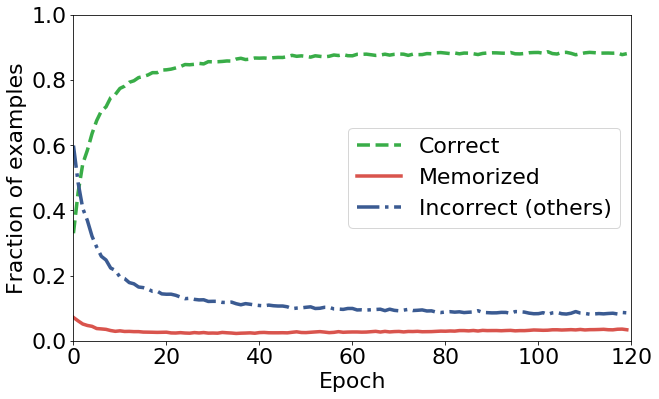}
  \end{tabular}

    \caption{Results of training a ResNet-34~\citep{he2016deep} neural network with a traditional cross entropy loss (top row) and our proposed method (bottom row) to perform classification on the CIFAR-10 dataset where 40\% of the labels are flipped at random. The left column shows the fraction of examples with clean labels that are predicted correctly (green) and incorrectly (blue). The right column shows the fraction of examples with wrong labels that are predicted correctly (green), \emph{memorized} (the prediction equals the wrong label, shown in red), and incorrectly predicted as neither the true nor the labeled class (blue). The model trained with cross entropy begins by learning to predict the true labels, even for many of the examples with wrong label, but eventually memorizes the wrong labels. Our proposed method based on early-learning regularization prevents memorization, allowing the model to continue learning on the examples with clean labels to attain high accuracy on examples with both clean and wrong labels.}
    \label{fig:CE_vs_ELR}
\end{figure}

\vspace{-1mm}
\section{Related Work}
\label{sec:related_work}
In this section we describe existing techniques to train deep-learning classification models using data with noisy annotations. 
We focus our discussion on methods that do not assume the availability of small subsets of training data with clean labels (as opposed, for example, to~\citep{Hendrycks2018UsingTD,Ren2018LearningTR,veit2017learning}). We also assume that the correct classes are known (as opposed to~\citep{Wang2018IterativeLW}). 

\emph{Robust-loss} methods propose cost functions specifically designed to be robust in the presence of noisy labels. These include Mean Absolute Error (MAE)~\citep{ghosh2017robust}, Improved MAE~\citep{wang2019imae}, which is a reweighted MAE, Generalized Cross Entropy~\citep{zhang2018generalized}, which can be interpreted as a generalization of MAE, Symmetric Cross Entropy~\citep{Wang2019SymmetricCE}, which adds a reverse cross-entropy term to the usual cross-entropy loss, and $\mathcal{L}_{\text{DIM}}$~\citep{Xu2019L_DMIAN}, which is based on information-theoretic considerations. \emph{Loss-correction} methods explicitly correct the loss function to take into account the noise distribution, represented by a transition matrix of mislabeling probabilities~\citep{patrini2017making,Goldberger2017TrainingDN,xia2019anchor, tanno2019learning}. 

Robust-loss and loss-correction techniques do not exploit the early-learning phenomenon mentioned in the introduction. This phenomenon was described in~\citep{arpit2017closer} (see also~\citep{zhang2016understanding}), and analyzed theoretically in~\citep{li2019gradient}. Our theoretical approach differs from theirs in two respects. First, Ref.~\citep{li2019gradient} focus on a least squares regression task, whereas we focus on the noisy label problem in classification. Second, and more importantly, we prove that early learning and memorization occur even in a \emph{linear} model.

Early learning can be exploited through \emph{sample selection}, where the model output during the early-learning stage is used to predict which examples are mislabeled and which have been labeled correctly. The prediction is based on the observation that mislabeled examples tend to have higher loss values. Co-teaching~\citep{Han2018CoteachingRT,Yu2019HowDD} performs sample selection by using two networks, each trained on a subset of examples that have a small training loss for the other network (see~\citep{Jiang2018MentorNetLD,malach2017decoupling} for related approaches). A limitation of this approach is that the examples that are selected tend to be \emph{easier}, in the sense that the model output during early learning approaches the true label. As a result, the gradient of the cross-entropy with respect to these examples is small, which slows down learning~\citep{chang2017active}. In addition, the subset of selected examples may not be rich enough to generalize effectively to held-out data~\citep{song2019selfie}.

An alternative to sample selection is \emph{label correction}. During the early-learning stage the model predictions are accurate on a subset of the mislabeled examples (see the top row of Figure~\ref{fig:CE_vs_ELR}). This suggests correcting the corresponding labels. This can be achieved by computing new labels equal to the probabilities estimated by the model (known as \emph{soft labels}) or to one-hot vectors representing the model predictions (\emph{hard labels})~\citep{Tanaka2018JointOF,yi2019probabilistic}. Another option is to set the new labels to equal a convex combination of the noisy labels and the soft or hard labels~\citep{Reed2015TrainingDN}. Label correction is usually combined with some form of iterative sample selection~\citep{Arazo2019unsup,Ma2018DimensionalityDrivenLW,song2019selfie,li2020dividemix} or with additional regularization terms~\citep{Tanaka2018JointOF}.  SELFIE~\citep{song2019selfie} uses label replacement to correct a subset of the labels selected by considering past model outputs. Ref.~\citep{Ma2018DimensionalityDrivenLW} computes a different convex combination with hard labels for each example based on a measure of model dimensionality. Ref.~\citep{Arazo2019unsup} fits a two-component mixture model to carry out sample selection, and then corrects labels via convex combination as in \citep{Reed2015TrainingDN}. They also apply mixup data augmentation~\citep{zhang2017mixup} to enhance performance. In a similar spirit, DivideMix~\citep{li2020dividemix} uses two networks to perform sample selection via a two-component mixture model, and applies the semi-supervised learning technique MixMatch~\citep{berthelot2019mixmatch}.  

Our proposed approach is somewhat related in spirit to label correction. We compute a probability estimate that is analogous to the soft labels mentioned above, and then exploit it to avoid memorization. However it is also fundamentally different: instead of modifying the labels, we propose a novel regularization term explicitly designed to correct the gradient of the cross-entropy cost function. This yields strong empirical performance, without needing to incorporate sample selection.

\vspace{-1mm}
\section{Early learning as a general phenomenon of high-dimensional classification}
\label{sec:linear}

As the top row of Figure~\ref{fig:CE_vs_ELR} makes clear, deep neural networks trained with noisy labels make progress during the early learning stage before memorization occurs.
In this section, we show that far from being a peculiar feature of deep neural networks, this phenomenon is intrinsic to high-dimensional classification tasks, even in the simplest setting.
Our theoretical analysis is also the inspiration for the early-learning regularization procedure we propose in Section~\ref{sec:methodology}.

We exhibit a simple \emph{linear} model with noisy labels which evinces the same behavior as described above: the \emph{early learning} stage, when the classifier learns to correctly predict the true labels, even on noisy examples, and the \emph{memorization} stage, when the classifier begins to make incorrect predictions because it memorizes the wrong labels. This is illustrated in Figure~\ref{fig:CE_vs_ELR_linear}, which demonstrates that empirically the linear model has the same qualitative behavior as the deep-learning model in Figure~\ref{fig:CE_vs_ELR}.
We show that this behavior arises because, early in training, the gradients corresponding to the correctly labeled examples dominate the dynamics---leading to early progress towards the true optimum---but that the gradients corresponding to wrong labels soon become dominant---at which point the classifier simply learns to fit the noisy labels.

We consider data drawn from a mixture of two Gaussians in $\mathbb R^p$.
The (clean) dataset consists of $n$ i.i.d.~copies of $(\mathbf x, \vytrue)$. The label $\vytrue \in \{0, 1\}^2$ is a one-hot vector representing the cluster assignment, and 
\begin{align*}
    \mathbf x & \sim \mathcal N(+\vv , \sigma^2 I_{p \times p}) \quad \text{ if $\vytrue = (1, 0)$} \\
    \mathbf x & \sim \mathcal N(-\vv, \sigma^2 I_{p \times p}) \quad \text{ if $\vytrue = (0, 1) $}\,,
\end{align*}
where $\vv$ is an arbitrary unit vector in $\mathbb R^p$ and $\sigma^2$ is a small constant.
The optimal separator between the two classes is a hyperplane through the origin perpendicular to $\vv$.
We focus on the setting where $\sigma^2$ is fixed while $n, p \to \infty$.
In this regime, the classification task is nontrivial, since the clusters are, approximately, two spheres whose centers are separated by 2 units with radii $\sigma \sqrt{p} \gg 2$.

We only observe a dataset with noisy labels $(\vy^{[1]}, \dots, \vy^{[n]})$, 
\begin{align}
    \vy^{[i]} = \left\{\begin{array}{ll}
    (\vytrue)^{[i]} & \text{with probability $1 - \noise$} \\
    \tilde \vy^{[i]} & \text{with probability $\noise$,}
    \end{array}\right.
    \label{eq:sym_noise}
\end{align}
where $\{\tilde \vy^{[i]}\}_{i=1}^n$ are i.i.d.~random one-hot vectors which take values $(1, 0)$ and $(0, 1)$ with equal probability.

We train a linear classifier by gradient descent on the cross entropy:
\begin{equation*}
    \min_{\param \in \mathbb R^{2\times p}} \mathcal L_\text{CE}(\param) := - \frac 1n \sum_{i=1}^n \sum_{c = 1}^2 \vy^{[i]}_c \log (\smax(\param \vx^{[i]})_c)\,,
\end{equation*}
where $\smax: \mathbb R^2 \to [0, 1]^2$ is a softmax function. 
In order to separate the true classes well (and not overfit to the noisy labels), the rows of $\param$ should be correlated with the vector $\vv$.

The gradient of this loss with respect to the model parameters $\param$ corresponding to class $c$ reads
\begin{align}\label{eq:linear_gradient}
\nabla \mathcal{L}_\text{CE}(\param)_c & = \frac{1}{n}\sum_{i=1}^n \vx^{[i]} \left(  \smax(\param \vx^{[i]})_c - \vy^{[i]}_c\right), 
\end{align}
Each term in the gradient therefore corresponds to a weighted sum of the examples $\vx^{[i]}$, where the weighting depends on the agreement between $\smax(\param \vx^{[i]})_c$ and $\vy^{[i]}_c$.

Our main theoretical result shows that this linear model possesses the properties described above.
During the early-learning stage, the algorithm makes progress and the accuracy on wrongly labeled examples increases.
However, during this initial stage, the relative importance of the wrongly labeled examples continues to grow; once the effect of the wrongly labeled examples begins to dominate, memorization occurs.

\begin{theorem}[Informal]\label{thm:main}
Denote by $\{\param_t\}$ the iterates of gradient descent with step size $\eta$.
For any $\Delta \in (0, 1)$, there exists a constant $\sigma_\Delta$ such that, if $\sigma \leq \sigma_\Delta$
and $p/n \in (1-\Delta/2, 1)$, then with probability $1 - o(1)$ as $n, p \to \infty$ there exists a $T = \Omega(1/\eta)$ such that:
\begin{itemize}
\item \textbf{Early learning succeeds}:  For $t < T$, $-\nabla \mathcal L(\param_t)$ is well correlated with the correct separator $\vv$, and at $t = T$ the classifier has higher accuracy on the wrongly labeled examples than at initialization.
\item \textbf{Gradients from correct examples vanish}: Between $t = 0$ and $t = T$, the magnitudes of the coefficients $\left(  \smax(\param_t \vx^{[i]})_c - \vy^{[i]}_c\right)$ corresponding to examples with clean labels decreases while the magnitudes of the coefficients for examples with wrong labels increases. 
\item \textbf{Memorization occurs}: As $t \to \infty$, the classifier $\param_t$ memorizes all noisy labels.
\end{itemize}
\end{theorem}
Due to space constraints, we defer the formal statement of Theorem~\ref{thm:main} and its proof to the supplementary material.

The proof of Theorem~\ref{thm:main} is based on two observations. First, while $\param$ is still not well correlated with $\vv$, the coefficients $\smax(\param \vx^{[i]})_c - \vy^{[i]}_c$ are similar for all $i$, so that $\nabla \mathcal L_\text{CE}$ points approximately in the average direction of the examples. Since the majority of data points are correctly labeled, this means the gradient is still well correlated with the correct direction during the early learning stage. Second, once $\param$ becomes correlated with $\vv$, the gradient begins to point in directions orthogonal to the correct direction $\vv$; when the dimension is sufficiently large, there are enough of these orthogonal directions to allow the classifier to completely memorize the noisy labels. 

This analysis suggests that in order to learn on the correct labels and avoid memorization it is necessary to (1) ensure that the contribution to the gradient from examples with clean labels remains large, and (2) neutralize the influence of the examples with wrong labels on the gradient. In Section 4 we propose a method designed to achieve this via regularization.

\section{Methodology}\label{sec:methodology}

\subsection{Gradient analysis of softmax classification from noisy labels}
\label{sec:motivation}
In this section we explain the connection between the linear model from Section~\ref{sec:linear} and deep neural networks. Recall the gradient of the cross-entropy loss with respect to $\param$ given in~\eqref{eq:linear_gradient}.
Performing gradient descent modifies the parameters iteratively to push $\smax(\param \vx^{[i]})$ closer to $\vy^{[i]}$. If $c$ is the true class so that $\vy^{[i]}_c=1$, the contribution of the $i$th example to $\nabla \mathcal{L}_\text{CE}(\param)_c$ is aligned with $-\vx^{[i]}$, and gradient descent moves in the direction of $\vx^{[i]}$. However, if the label is noisy and $\vy^{[i]}_c=0$, then gradient descent moves in the opposite direction, which eventually leads to memorization as established by Theorem~\ref{thm:main}. 

We now show that for nonlinear models based on neural networks, the effect of label noise is analogous. We consider a classification problem with $C$ classes, where the training set consists of $n$ examples $\{\vx^{[i]}, \vy^{[i]}\}_{i=1}^n$, $\vx^{[i]} \in \R^d$ is the $i$th input and $\vy^{[i]}\in \{0,1\}^C$ is a one-hot label vector indicating the corresponding class. The classification model maps each input $\vx^{[i]}$ to a $C$-dimensional encoding using a deep neural network $\func_{\vx^{[i]}}(\param)$ and then feeds the encoding into a softmax function $\ml{S}$ to produce an estimate $\vp^{[i]}$ of the conditional probability of each class given $\vx^{[i]}$, 
\begin{align}
\vp^{[i]} := \smax \left( \func_{\vx^{[i]}}(\param) \right).
\end{align}
$\param$ denotes the parameters of the neural network. The gradient of the cross-entropy loss,
\begin{align} 
\mathcal{L}_\text{CE}(\param) &:= -\frac{1}{n}\sum_{i=1}^n\sum_{c=1}^{C} \vy^{[i]}_c \log \vp^{[i]}_c,
\end{align}
with respect to $\param$ equals
\begin{align} 
\nabla \mathcal{L}_\text{CE}(\param) 
 & = \frac{1}{n}\sum_{i=1}^n \nabla \func_{\vx^{[i]}}(\param) \left(  \vp^{[i]} - \vy^{[i]}\right),
 \label{eq:ce_grad}
\end{align}
where $\nabla \func_{\vx^{[i]}}(\param)$ is the Jacobian matrix of the neural-network encoding for the $i$th input with respect to $\param$. Here we see that label noise has the same effect as in the simple linear model. If $c$ is the true class, but $\vy^{[i]}_c=0$ due to the noise, then the contribution of the $i$th example to $\nabla \mathcal{L}_\text{CE}(\param)_c$ is reversed. The entry corresponding to the \emph{impostor} class $c'$, is also reversed because $\vy^{[i]}_{c'}=1$. As a result, performing stochastic gradient descent eventually results in memorization, as in the linear model (see Figures~\ref{fig:CE_vs_ELR} and~
\ref{fig:CE_vs_ELR_linear}). Crucially, the influence of the label noise on the gradient of the cross-entropy loss is restricted to the term $\vp^{[i]} - \vy^{[i]}$ (see Figure~\ref{fig:Gradient_CE}). In Section~\ref{sec:ELR} we describe how to counteract this influence by exploiting the early-learning phenomenon.

\subsection{Early-learning regularization}
\label{sec:ELR}
In this section we present a novel framework for learning from noisy labels called early-learning regularization (ELR). We assume that we have available a \emph{target}\footnote{The term target is inspired by semi-supervised learning where target probabilities are used to learn on unlabeled examples~\citep{yarowsky1995unsupervised, mcclosky2006effective, laine2016temporal}.} vector of probabilities $\vq^{[i]}$ for each example $i$, which is computed using past outputs of the model. Section~\ref{sec:targets} describes several techniques to compute the targets. Here we explain how to use them to avoid memorization. 

Due to the early-learning phenomenon, we assume that at the beginning of the optimization process the targets do not overfit the noisy labels. ELR exploits this using a regularization term that seeks to maximize the inner product between the model output and the targets,
\begin{align} 
\mathcal{L}_\text{ELR}(\param) &:=\mathcal{L}_\text{CE}(\param) +  \frac{ \lambda}{n}\sum_{i=1}^n \log \left(1-\langle \vp^{[i]},\vq^{[i]} \rangle\right). \label{eq:ELR}
\end{align}
The logarithm in the regularization term counteracts the exponential function implicit in the softmax function in $\vp^{[i]}$. A possible alternative to this approach would be to penalize the Kullback-Leibler divergence between the model outputs and the targets. However, this does not exploit the early-learning phenomenon effectively, because it leads to overfitting the targets as demonstrated in Section~\ref{sec:kl}. 

The key to understanding why ELR is effective lies in its gradient, derived in the following lemma, which is proved in Section~\ref{sec:proof_lemma}.

\begin{lemma}[Gradient of the ELR loss]
\label{lemma:ELR_gradient}
The gradient of the loss defined in Eq.~\eqref{eq:ELR} is equal to
\begin{align} 
\nabla \mathcal{L}_\text{ELR}(\param) & = \frac{1}{n}\sum_{i=1}^n \nabla \func_{\vx^{[i]}}(\param) \brac{ \vp^{[i]} - \vy^{[i]} + \lambda \vg^{[i]} }
\end{align}
where the entries of $\vg^{[i]} \in \R^{C}$ are given by
\begin{align}
\vg^{[i]}_c := \frac{\vp^{[i]}_{c}}{1-\langle \vp^{[i]},\vq^{[i]} \rangle}\sum_{k=1}^C(\vq_{k}^{[i]}-\vq_{c}^{[i]})\vp_{k}^{[i]}, \qquad 1\leq c \leq C.
\end{align}
\end{lemma}

\begin{figure}[t]
\hspace{-0.4cm}
    \begin{tabular}{>{\centering\arraybackslash}m{0.31\linewidth} >{\centering\arraybackslash}m{0.31\linewidth} >{\centering\arraybackslash}m{0.31\linewidth}}
    \quad\quad$\vp_{c^{\ast}}^{[i]} - \vy_{c^{\ast}}^{[i]} + \lambda \vg_{c^{\ast}}^{[i]}$ & \quad\quad Clean labels & \quad\quad Wrong labels  \\
      \includegraphics[width=1.1\linewidth]{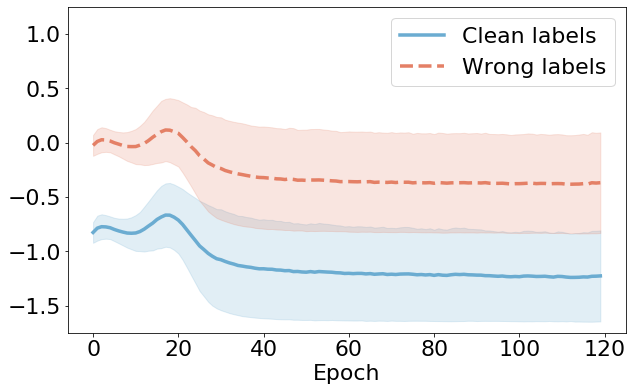}
      &
     \includegraphics[width=1.1\linewidth]{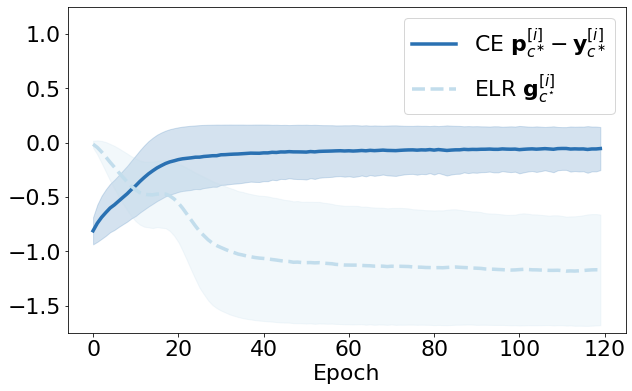}&
     \includegraphics[width=1.1\linewidth]{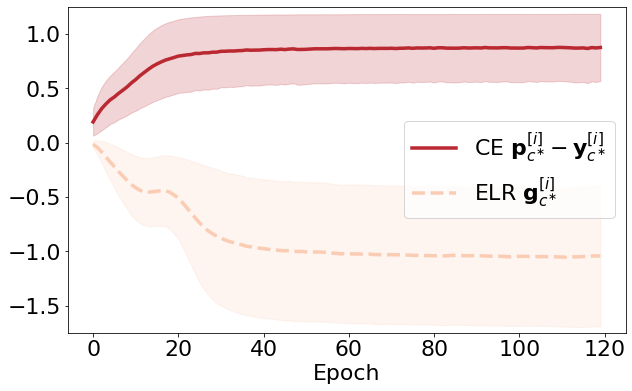}
    \end{tabular}
    \caption{
    Illustration of the effect of the regularization on the gradient of the ELR loss (see Lemma~\ref{lemma:ELR_gradient}) for the same deep-learning model as in Figure~\ref{fig:CE_vs_ELR}. On the left, we plot the entry of $\vp^{[i]} - \vy^{[i]} + \lambda \vg^{[i]}$ corresponding to the true class, denoted by $c^{\ast}$, for training examples with clean (blue) and wrong (red) labels. The center image shows the $c^{\ast}$th entry of the cross-entropy (CE) term $\vp^{[i]} - \vy^{[i]}$ (dark blue) and the regularization term $\vg^{[i]}$ (light blue) separately for the examples with clean labels. During early learning the CE term dominates, but afterwards it vanishes as the model learns the clean labels (i.e. $\vp^{[i]} \approx \vy^{[i]}$). However, the regularization term compensates for this, forcing the model to continue learning mainly on the examples with clean labels. On the right, we show the CE and the regularization term (dark and light red respectively) separately for the examples with wrong labels. The regularization cancels out the CE term, preventing memorization. In all plots the curves represent the mean value, and the shaded regions are within one standard deviation of the mean.
    }
    \label{fig:Gradient_simple}
\end{figure}

In words, the sign of $\vg^{[i]}_c$ is determined by a weighted combination of the difference between $\vq_c^{[i]}$ and the rest of the entries in the target. 

If $c^{\ast}$ is the true class, then the $c^{\ast}$th entry of $\vq^{[i]}$ tends to be dominant during early-learning. In that case, the $c^{\ast}$th entry of $\vg^{[i]}$ is negative.
This is useful both for examples with clean labels and for those with wrong labels.
For examples with clean labels, the cross-entropy term $\vp^{[i]} - \vy^{[i]}$ tends to vanish after the early-learning stage because $\vp^{[i]}$ is very close to $\vy^{[i]}$, allowing examples with wrong labels to dominate the gradient. Adding $\vg^{[i]}$ counteracts this effect by ensuring that the magnitudes of the coefficients on examples with clean labels remains large. The center image of Figure~\ref{fig:Gradient_simple} shows this effect. For examples with wrong labels, the cross entropy term $\vp^{[i]}_{c^\ast} - \vy_{c^\ast}^{[i]}$ is positive because $\vy_{c^\ast}^{[i]} = 0$. Adding the negative term $\vg^{[i]}_{c^{\ast}}$ therefore dampens the coefficients on these mislabeled examples, thereby diminishing their effect on the gradient (see right image in Figure~\ref{fig:Gradient_simple}). Thus, ELR fulfils the two desired properties outlined at the end of Section~\ref{sec:linear}: boosting the gradient of examples with clean labels, and neutralizing the gradient of the examples with false labels.

\subsection{Target estimation}
\label{sec:targets}
ELR requires a target probability for each example in the training set. The target can be set equal to the model output, but using a running average is more effective. In semi-supervised learning, this technique is known as temporal ensembling~\citep{laine2016temporal}. Let $\vq^{[i]}(k)$ and $\vp^{[i]}(k)$ denote the target and model output respectively for example $i$ at iteration $k$ of training. We set
\begin{align}
    \vq^{[i]}(k) := \beta \vq^{[i]}(k-1) + (1-\beta) \vp^{[i]}(k),
    \label{eq:moving_average}
\end{align}
where $0 \leq \beta < 1$ is the momentum. The basic version of our proposed method alternates between computing the targets and minimizing the cost function~\eqref{eq:ELR} via stochastic gradient descent.

Target estimation can be further improved in two ways. First, by using the output of a model obtained through a running average of the model weights during training. In semi-supervised learning, this \emph{weight averaging} approach has been proposed to mitigate confirmation bias~\citep{tarvainen2017mean}. Second, by using two separate neural networks, where the target of each network is computed from the output of the other network. The approach is inspired by Co-teaching and related methods~\citep{Han2018CoteachingRT,Yu2019HowDD,li2020dividemix}. The ablation results in Section~\ref{sec:results} show that weight averaging, two networks, and mixup data augmentation~\citep{zhang2017mixup} all separately improve performance. We call the combination of all these elements ELR+. A detailed description of ELR and ELR+ is provided in Section~\ref{sec:algorithms} of the supplementary material.

\begin{table}
\footnotesize
\begin{center}

\begin{threeparttable}
\resizebox{\linewidth}{!}{
\begin{tabular}{c|c|cccc|cccc}
\toprule
\multirow{2}{*}{\shortstack{Datasets\\(Architecture)}} & \multirow{2}{*}{Methods} & \multicolumn{4}{c|}{\textit{Symmetric label noise}}&\multicolumn{4}{c}{\textit{Asymmetric label noise}} \\
 & & 20\% & 40\%  & 60\% & 80\%  & 10\% & 20\% & 30\% & 40\%\\
\midrule
\multirow{7}{*}{\shortstack{CIFAR10\\(ResNet34)}}  & Cross entropy & 86.98 $\pm$ 0.12 & 81.88 $\pm$ 0.29& 74.14 $\pm$ 0.56 & 53.82 $\pm$ 1.04 & 90.69 $\pm$ 0.17 & 88.59 $\pm$ 0.34 & 86.14 $\pm$ 0.40 & 80.11 $\pm$ 1.44 \\
   & Bootstrap~\cite{Reed2015TrainingDN}& 86.23 $\pm$ 0.23 & 82.23 $\pm$ 0.37  & 75.12 $\pm$ 0.56 & 54.12 $\pm$ 1.32 & 90.32 $\pm$ 0.21 & 88.26 $\pm$ 0.24 & 86.57 $\pm$ 0.35 & 81.21 $\pm$ 1.47\\
 & Forward~\cite{patrini2017making}& 87.99 $\pm$ 0.36 & 83.25 $\pm$ 0.38 & 74.96 $\pm$ 0.65 & 54.64 $\pm$ 0.44 & 90.52 $\pm$ 0.26 & 89.09 $\pm$ 0.47 & 86.79 $\pm$ 0.36& 83.55 $\pm$ 0.58\\
  & GSE~\cite{zhang2018generalized}& 89.83 $\pm$ 0.20 & 87.13 $\pm$ 0.22 & 82.54 $\pm$ 0.23 & 64.07 $\pm$ 1.38 & 90.91 $\pm$ 0.22 & 89.33 $\pm$ 0.17 & 85.45 $\pm$ 0.74 & 76.74 $\pm$ 0.61\\
& SL~\cite{Wang2019SymmetricCE} & 89.83 $\pm$ 0.32 & 87.13 $\pm$ 0.26 & 82.81 $\pm$ 0.61 & 68.12 $\pm$ 0.81 & 91.72 $\pm$ 0.31 & 90.44 $\pm$ 0.27 & 88.48 $\pm$ 0.46 & 82.51 $\pm$ 0.45\\
 & ELR & \textbf{91.16 $\pm$ 0.08} & \textbf{89.15 $\pm$ 0.17 } & \textbf{86.12 $\pm$ 0.49} & \textbf{73.86 $\pm$ 0.61}  &\textbf{93.27 $\pm$ 0.11} & \textbf{93.52 $\pm$ 0.23} & \textbf{91.89 $\pm$ 0.22} & \textbf{90.12 $\pm$ 0.47} \\
  & ELR\tnote{$\star$} & \textbf{ 92.12 $\pm$ 0.35} & \textbf{91.43 $\pm$ 0.21} & \textbf{88.87 $\pm$ 0.24} & \textbf{80.69 $\pm$ 0.57}& \textbf{94.57 $\pm$ 0.23} & \textbf{93.28 $\pm$ 0.19} & \textbf{92.70 $\pm$ 0.41} & \textbf{90.35 $\pm$ 0.38} \\
\midrule
\multirow{7}{*}{\shortstack{CIFAR100\\(ResNet34)}} & Cross entropy & 58.72 $\pm$ 0.26 & 48.20 $\pm$ 0.65 & 37.41 $\pm$ 0.94 & 18.10 $\pm$ 0.82 &  66.54 $\pm$ 0.42 & 59.20 $\pm$ 0.18 & 51.40 $\pm$ 0.16 & 42.74 $\pm$ 0.61 \\
 & Bootstrap~\cite{Reed2015TrainingDN}& 58.27 $\pm$ 0.21 & 47.66 $\pm$ 0.55 & 34.68 $\pm$ 1.1 & 21.64 $\pm$ 0.97 & 67.27 $\pm$ 0.78 & 62.14 $\pm$ 0.32 & 52.87 $\pm$ 0.19 & 45.12 $\pm$ 0.57\\
  & Forward~\cite{patrini2017making} & 39.19 $\pm$ 2.61 & 31.05 $\pm$ 1.44 & 19.12 $\pm$ 1.95 & 8.99 $\pm$ 0.58 & 45.96 $\pm$ 1.21 & 42.46 $\pm$ 2.16 & 38.13 $\pm$ 2.97 & 34.44 $\pm$ 1.93\\
  & GSE~\cite{zhang2018generalized} & 66.81 $\pm$ 0.42 & 61.77 $\pm$ 0.24 & 53.16 $\pm$ 0.78 & 29.16 $\pm$ 0.74 & 68.36 $\pm$ 0.42 & 66.59 $\pm$ 0.22 & 61.45 $\pm$ 0.26 & 47.22 $\pm$ 1.15\\
  & SL~\cite{Wang2019SymmetricCE} & 70.38 $\pm$ 0.13 & 62.27 $\pm$ 0.22 & 54.82 $\pm$ 0.57 & 25.91 $\pm$ 0.44 & 73.12 $\pm$ 0.22 & 72.56 $\pm$ 0.22 &72.12 $\pm$ 0.24 & 69.32 $\pm$ 0.87\\
  & ELR  & \textbf{74.21 $\pm$ 0.22} & \textbf{68.28 $\pm$ 0.31} & \textbf{59.28 $\pm$ 0.67} & \textbf{29.78 $\pm$ 0.56} & 
                     \textbf{74.20 $\pm$ 0.31} & \textbf{74.03 $\pm$ 0.31} & \textbf{73.71 $\pm$ 0.22} & \textbf{73.26 $\pm$ 0.64}\\
  & ELR\tnote{$\star$}  & \textbf{74.68 $\pm$ 0.31} & \textbf{68.43 $\pm$ 0.42} & \textbf{60.05 $\pm$ 0.78} & \textbf{30.27 $\pm$ 0.86} & 
                      \textbf{74.52 $\pm$ 0.32} & \textbf{74.20 $\pm$ 0.25} & \textbf{74.02 $\pm$ 0.33} & \textbf{73.73 $\pm$ 0.34}\\
\bottomrule
\end{tabular}}
\begin{tablenotes}\footnotesize
\item[$\star$] \textit{Results with cosine annealing learning rate. }
\end{tablenotes}
\end{threeparttable}
\end{center}

\caption{Comparison with state-of-the-art methods on CIFAR-10 and CIFAR-100 with symmetric and asymmetric label noise. The bootstrap and SL methods were reimplemented using publicly available code, the rest of results are taken from~\cite{zhang2018generalized}. The mean accuracy and its standard deviation are computed over five noise realizations.}
\label{tab:results_loss_correction_methods}
\vspace{-5mm}
\end{table}

\section{Experiments}
We evaluate the proposed methodology on two standard benchmarks with simulated label noise, CIFAR-10 and CIFAR-100~\citep{krizhevsky2009learning}, and two real-world datasets, Clothing1M~\citep{xiao2015learning} and WebVision~\citep{li2017webvision}. For CIFAR-10 and CIFAR-100 we simulate label noise by randomly flipping a certain fraction of the labels in the training set following a \textit{symmetric} uniform distribution (as in Eq.~(\ref{eq:sym_noise})), as well as a more realistic \textit{asymmetric} class-dependent distribution, following the scheme proposed in~\citep{patrini2017making}. Clothing1M consists of 1 million training images collected from online shopping websites with labels generated using surrounding text. Its noise level is estimated at $38.5\%$~\citep{song2019prestopping}. For ease of comparison to previous works~\citep{Jiang2018MentorNetLD,Chen2019UnderstandingAU}, we consider the mini WebVision dataset which contains the top 50 classes from the Google image subset of WebVision, which results in approximately 66 thousand images.
The noise level of WebVision is estimated at $20\%$~\citep{li2017webvision}. Table~\ref{tab:data_describ} in the supplementary material reports additional details about the datasets, and our training, validation and test splits. 

In our experiments, we prioritize making our results comparable to the existing literature. When possible we use the same preprocessing, and architectures as previous methods. The details are described in Section~\ref{sec:experiments_appendix} of the supplementary material. We focus on two variants of the proposed approach: ELR with temporal ensembling, which we call ELR, and ELR with temporal ensembling, weight averaging, two networks, and mixup data augmentation, which we call ELR+ (see Section~\ref{sec:algorithms}). The choice of hyperparameters is performed on separate validation sets. Section~
\ref{sec:hyperparameters} shows that the sensitivity to different hyperparameters is quite low. Finally, we also perform an ablation study on CIFAR-10 for two levels of symmetric noise (40\% and 80\%) in order to evaluate the contribution of the different elements in ELR+. Code to reproduce the experiments is publicly available online at \url{https://github.com/shengliu66/ELR}.

\vspace{-1mm}
\section{Results}
\label{sec:results}
\begin{table}
\begin{threeparttable}[t]
\resizebox{\linewidth}{!}{
\begin{tabular}{c|c|c|c|c|c|c|c|c|c|c}
\toprule
& & &
Cross entropy & 
Co-teaching+~\citep{Yu2019HowDD} & 
Mixup~\citep{zhang2017mixup} &
PENCIL~\citep{yi2019probabilistic} & 
MD-DYR-SH~\citep{Arazo2019unsup} &
DivideMix~\citep{li2020dividemix} & 
ELR+ &
ELR+\tnote{$\ast$}\\ 
\midrule
\multirow{5}{*}{CIFAR-10}&\multirow{4}{*}{\shortstack{Sym.\\label\\ noise}
}&{\large 20\%} &  {\large 86.8} &  {\large 89.5} & {\large 95.6} & {\large 92.4} & {\large 94.0} & {\large \textbf{96.1}} & {\large 94.6}
&{\large {95.8}}\\
&&{\large 50\%} & {\large 79.4} & {\large 85.7} & {\large 87.1} & {\large 89.1} & {\large 92.0} & {\large {94.6}} & {\large 93.8} &{\large \textbf{94.8}} \\
&&{\large 80\%} & {\large 62.9}  & {\large 67.4} & {\large 71.6} & {\large 77.5} & {\large 86.8} & {\large {93.2}} & {\large 91.1}
&{\large \textbf{93.3}}\\
&&{\large 90\%} & {\large 42.7}  & {\large 47.9} & {\large 52.2} & {\large 58.9} & {\large 69.1} & {\large {76.0}} & {\large 75.2}
&{\large \textbf{78.7}}\\
\cmidrule{2-11}
&{Asym.}&{\large 40\%} & {\large 83.2}  & {-} & {-} & {\large 88.5} & {\large 87.4} & {\large \textbf{93.4}} & {\large 92.7} & {\large 93.0}\\
\bottomrule
\toprule
\multirow{5}{*}{CIFAR-100}&\multirow{4}{*}{\shortstack{Sym.\\label\\ noise}
}&{\large 20\%} &  {\large 62.0} &  {\large 65.6} & {\large 67.8} & {\large 69.4} & {\large 73.9} & {\large 77.3} & {\large {77.5}} & {\large \textbf{77.6}}\\
&&{\large 50\%} & {\large 46.7} & {\large 51.8} &
{\large 57.3}&
{\large 57.5} & {\large 66.1} &  {\large \textbf{74.6}} & {\large 72.4} &  {\large {73.6}}\\
&&{\large 80\%} & {\large 19.9}  & {\large 27.9} & {\large 30.8} & {\large 31.1} & {\large 48.2} & {\large 60.2} & {\large 58.2} & {\large \textbf{60.8}}\\
&&{\large 90\%} & {\large 10.1}  & {\large 13.7} & {\large 14.6} & {\large 15.3} & {\large 24.3} & {\large 31.5} & {\large 30.8} & {\large \textbf{33.4}}\\
\cmidrule{2-11}
&{Asym.}&{\large 40\%} & {\large -}  & {-} & {-} & {\large -} & {\large -} & {\large 72.1 } & {\large {76.5}} &{\large \textbf{77.5}}\\
\bottomrule
\end{tabular}}
\end{threeparttable}
\vspace{0.5mm}
\caption{Comparison with state-of-the-art methods 
on CIFAR-10 and CIFAR-100 with symmetric and asymmetric noise. For ELR+, we use 10\% of the training set for validation, and treat the validation set as a held-out test set. The result for DivideMix on CIFAR-100 with 40\% asymmetric noise was obtained using publicly available code. The rest of the results are taken from~\citep{li2020dividemix}, which reports the highest accuracy observed on the validation set during training. We also report the performance of ELR+ under this metric on the rightmost column (ELR+$^{\ast}$).}
    \label{tab:Results_ELR_plus_CIFAR10}
\vspace{-1.5mm}
\end{table}

Table~\ref{tab:results_loss_correction_methods} evaluates the performance of ELR on CIFAR-10 and CIFAR-100 with different levels of symmetric and asymmetric label noise. We compare to the best performing methods that only modify the training loss. All techniques use the same architecture (ResNet34), batch size, and training procedure. ELR consistently outperforms the rest by a significant margin. To illustrate the influence of the training procedure, we include results with a different learning-rate scheduler (cosine annealing~\citep{Loshchilov2017SGDRSG}), which further improves the results. 

In Table~\ref{tab:Results_ELR_plus_CIFAR10}, we compare ELR+ to state-of-the-art methods, which also apply sample selection and data augmentation, on CIFAR-10 and CIFAR-100. All methods use the same architecture (\textbf{PreAct ResNet-18}). The results from other methods may not be completely comparable to ours because they correspond to the best test performance during training, whereas we use a separate validation set. Nevertheless, ELR+ outperforms all other methods except DivideMix.

\begin{table}[!h]
\resizebox{\linewidth}{!}{
\begin{tabular}{c|c|c|c|c|c|c|c}
\toprule
CE&
Forward~\citep{patrini2017making} & 
GCE~\citep{zhang2018generalized} &
SL~\citep{Wang2019SymmetricCE} &
Joint-Optim~\citep{Tanaka2018JointOF} &
DivideMix~\citep{li2020dividemix} &
ELR  &
ELR+ \\
\midrule
 69.10 &  69.84 &  69.75 &  71.02 &  72.16  &  74.76 &  72.87 &  \textbf{74.81}\\
\bottomrule
\end{tabular}}
\vspace{0.5mm}
\caption{Comparison with state-of-the-art methods in test accuracy (\%) on Clothing1M. All methods use a ResNet-50 architecture pretrained on ImageNet. Results of other methods are taken from the original papers (except for GCE, which is taken from~\citep{Wang2019SymmetricCE}).}
\vspace{-1.5mm}
\label{tab:clothing_results}
\end{table}
Table~\ref{tab:clothing_results} compares ELR and ELR+ to state-of-the-art methods on the Clothing1M dataset. ELR+ achieves state-of-the-art performance, slightly superior to DivideMix. 

\begin{table}[!h]
\resizebox{\linewidth}{!}{
	\begin{tabular}	{c|c|c|c|c|c|c|c|c}
		\toprule	 	
			& & 
			  D2L~\citep{Ma2018DimensionalityDrivenLW} &
			  MentorNet~\citep{Jiang2018MentorNetLD}&
			  Co-teaching~\citep{Han2018CoteachingRT} &	
			  Iterative-CV~\citep{Wang2018IterativeLW}	&	
			  DivideMix~\citep{li2020dividemix}&
			  ELR &
			  ELR+ \\
			  \midrule			
		\multirow{2}{*}{{\large WebVision}} & {\large top1} &  {\large 62.68} & {\large 63.00} & {\large 63.58} & {\large 65.24} & {\large 77.32} & {\large 76.26} & {\large \textbf{77.78}}\\
		& {\large top5}  & {\large 84.00} & {\large 81.40} & {\large 85.20} & {\large 85.34} & {\large 91.64} & {\large 91.26} & {\large\textbf{91.68}}\\
		 \midrule	
		 \multirow{2}{*}{{\large ILSVRC12}} & {\large top1}  & {\large 57.80} & {\large 57.80} & {\large 61.48} & {\large 61.60} & {\large \textbf{75.20}} & {\large 68.71} & {\large 70.29}\\
		& {\large top5}  & {\large 81.36} & {\large 79.92} & {\large 84.70} & {\large 84.98} & {\large\textbf{90.84}} & {\large 87.84} & {\large 89.76}\\

		\bottomrule
	\end{tabular}}
	\vspace{0.5mm}
    \caption{Comparison with state-of-the-art methods trained on the mini WebVision dataset. Results of other methods are taken from~\citep{li2020dividemix}. All methods use an InceptionResNetV2 architecture.
    \vspace{-1.5mm}}
    \label{tab:webvision_results}
\end{table}	
Table~\ref{tab:webvision_results} compares ELR and ELR+ to state-of-the-art methods trained on the mini WebVision dataset and evaluated on both the WebVision and ImageNet ILSVRC12 validation sets. ELR+ achieves state-of-the-art performance, slightly superior to DivideMix, on WebVision. ELR also performs strongly, despite its simplicity. On ILSVRC12 DivideMix produces superior results (particularly in terms of top1 accuracy).

\begin{table}[!h]
\begin{center}
\resizebox{0.85\linewidth}{!}{
\begin{tabular}{c c c c c cc}
\toprule
&  &  & \multicolumn{2}{c}{40\%} & \multicolumn{2}{c}{80\%}\\
\cmidrule{4-7}
&  &  & \multicolumn{2}{c}{Weight Averaging} & \multicolumn{2}{c}{Weight Averaging}\\
\cmidrule{4-7}
&  &  & \cmark & \xmark & \cmark & \xmark\\
\midrule[0.7pt]
\multirow{2}{*}{1 Network} &  \multirow{2}{*}{{mixup}} & \cmark & 93.04 $\pm$ 0.12 & 91.05 $\pm$ 0.13 & 87.23 $\pm$ 0.30 & 81.43 $\pm$ 0.52 \\
 \cmidrule{3-7}
& &\xmark & 92.09 $\pm$ 0.08 & 90.83 $\pm$ 0.07 & 76.50 $\pm$ 0.65 & 72.54 $\pm$ 0.35\\
\cmidrule[0.7pt]{1-7}
\multirow{2}{*}{2 Networks} &\multirow{2}{*}{{mixup}}& \cmark & 93.68 $\pm$ 0.51 & 93.51 $\pm$ 0.47 &  88.62 $\pm$ 0.26 & 84.75 $\pm$ 0.26 \\
\cmidrule{3-7}
& &\xmark & 92.95 $\pm$ 0.05 & 91.86 $\pm$ 0.14 & 80.13 $\pm$ 0.51 & 73.49 $\pm$ 0.47\\
\bottomrule
\end{tabular}}

\end{center}
\caption{Ablation study evaluating the influence of weight averaging, the use of two networks, and mixup data augmentation for the CIFAR-10 dataset with medium (40\%) and high (80\%) levels of symmetric noise. The mean accuracy and its standard deviation are computed over five noise realizations.
}
\label{tab:ablation_40}
\vspace{-2mm}
\end{table}
\vspace{-3mm}
Table~\ref{tab:ablation_40} shows the results of an ablation study evaluating the influence of the different elements of ELR+ for the CIFAR-10 dataset with medium (40\%) and high (80\%) levels of symmetric noise. Each element seems to provide an independent performance boost. At the medium noise level the improvement is modest, but at the high noise level it is very significant. This is in line with recent works showing the effectiveness of semi-supervised learning techniques in such settings~\citep{Arazo2019unsup,li2020dividemix}.

\section{Discussion and Future Work}
In this work we provide a theoretical characterization of the early-learning and memorization phenomena for a linear generative model, and build upon the resulting insights to propose a novel framework for learning from data with noisy annotations. Our proposed methodology yields strong results on standard benchmarks and real-world datasets for several different network architectures. However, there remain multiple open problems for future research. On the theoretical front, it would be interesting to bridge the gap between linear and nonlinear models (see~\citep{li2019gradient} for some work in this direction), and also to investigate the dynamics of the proposed regularization scheme. On the methodological front, we hope that our work will trigger interest in the design of new forms of regularization that provide robustness to label noise. 

\section{Broader Impact}

This work has the potential to advance the development of machine-learning methods that can be deployed in contexts where it is costly to gather accurate annotations. This is an important issue in applications such as medicine, where machine learning has great potential societal impact. 

\subsubsection*{Acknowledgments}
This research was supported by NSF NRT-HDR Award 1922658. SL was partially supported by NSF grant DMS 2009752. CFG was partially supported by NSF Award HDR-1940097. JNW gratefully acknowledges the support of the Institute for Advanced Study, where a portion of this research was conducted.

\medskip

\small
\bibliography{ref}
\bibliographystyle{plain}

\newpage

\appendix
\numberwithin{figure}{section}
\numberwithin{table}{section}
\section{Theoretical analysis of early learning and memorization in a linear model}
In this section, we formalize and substantiate the claims of Theorem~\ref{thm:main}.

Theorem~\ref{thm:main} has three parts, which we address in the following sections.
First, in Section~\ref{sec:early_learning_succeeds}, we show that the classifier makes progress during the early-learning phase: over the first $T$ iterations, the gradient is well correlated with $\vv$ and the accuracy on mislabeled examples increases.
However, as noted in the main text, this early progress halts because the gradient terms corresponding to correctly labeled examples begin to disappear.
We prove this rigorously in Section~\ref{sec:gradients}, which shows that the overall magnitude of the gradient terms corresponding to correctly labeled examples shrinks over the first $T$ iterations.
Finally, in Section~\ref{sec:memorization}, we prove the claimed asymptotic behavior: as $t \to \infty$, gradient descent perfectly memorizes the noisy labels.

\subsection{Notation and setup}
We consider a softmax regression model parameterized by two weight vectors $\param_1$ and $\param_2$, which are the rows of the parameter matrix $\param\in \R^{2\times p}$. In the linear case this is equivalent to a logistic regression model, because the cross-entropy loss on two classes depends only on the vector $\param_1 - \param_2$. 
If we reparametrize the labels as
\begin{equation*}
    \pmlabels^{[i]} = \left\{\begin{array}{ll}
    1 & \text{ if $\vy^{[i]}_1 = 1$} \\
    -1 & \text{ if $\vy^{[i]}_2 = 1$}\,,
    \end{array}\right.
\end{equation*}
and set $\theta := \param_1 - \param_2$, we can then write the loss as
\begin{equation*}
    \mathcal L_\text{CE}(\theta) = \frac 1n \sum_{i=1}^n  \log(1 + e^{-\pmlabels^{[i]}\theta^\top \vx^{[i]}})\,.
\end{equation*}
We write $\pmlabelstrue$ for the true cluster assignments: $(\pmlabelstrue)^{[i]} = 1$ if $\vx^{[i]}$ comes from the cluster with mean $+\vv$, and $(\pmlabelstrue)^{[i]} = -1$ otherwise.
Note that, with this convention, we can always write
$\vx^{[i]} = (\pmlabelstrue)^{[i]}(\vv - \sigma \vz^{[i]})$, where $\vz^{[i]}$ is a standard Gaussian random vector independent of all other random variables.

In terms of $\theta$ and $\pmlabels$, the gradient~\eqref{eq:linear_gradient} reads
\begin{align}
\nabla \mathcal{L}_\text{CE}(\theta) & = \frac{1}{2n}\sum_{i=1}^n \vx^{[i]} \left(\tanh(\theta^\top \vx^{[i]}) - \pmlabels^{[i]}\right), 
\end{align}
As noted in the main text, the coefficient $\tanh(\theta^\top \vx^{[i]}) - \pmlabels^{[i]}$ is the key quantity governing the properties of the gradient.

Let us write $\correct$ for the set of indices for which the labels are correct, and $W$ for the set of indices for which labels are wrong.

We assume that $\theta_0$ is initialized randomly on the sphere with radius $2$, and then optimized to minimize $\mathcal L$ via gradient descent with fixed step size $\eta < 1$. We denote the iterates by $\theta_t$.

We consider the asymptotic regime where $\sigma \ll 1$ and $\Delta$ are constants and $p, n \to \infty$, with $p/n \in (1 - \Delta/2, 1)$.
We will let $\sigma_\Delta$ denote a constant, whose precise value may change from proposition to proposition; however, in all cases the requirements on $\sigma$ will be \emph{independent} of $p$ and $n$.
For convenience, we assume that $\Delta \leq 1/2$, though it is straightforward to extend the analysis below to any $\Delta$ bounded away from $1$.
Note that when $\Delta = 1$, each observed label is independent of the data, so no learning is possible.
We will use the phrase ``with high probability'' to denote an event which happens with probability $1 - o(1)$ as $n, p \to \infty$, and we use $o_P(1)$ to denote a random quantity which converges to $0$ in probability.
We use the symbol $c$ to refer to an unspecified positive constant whose value may change from line to line.
We use subscripts to indicate when this constant depends on other parameters of the problem.

We let $T$ be the smallest positive integer such that $\theta_T^\top \vv \geq 1/10$.
By Lemmas~\ref{lem:init} and~\ref{lem:norm} in Section~\ref{sec:lemmas}, $T = \Omega(1/\eta)$ with high probability.

\subsection{Early-learning succeeds}\label{sec:early_learning_succeeds}
We first show that, for the first $T$ iterations, the negative gradient $- \nabla \mathcal L_{\text{CE}}(\theta_t)^\top$ has constant correlation with $\vv$.
(Note that, by contrast, a \emph{random} vector in $\R^p$ typically has \emph{negligible} correlation with $\vv$.)
\begin{proposition}\label{prop:correlation}
There exists a constant $\sigma_\Delta$, depending only on $\Delta$, such that if $\sigma \leq \sigma_\Delta$ then with high probability, for all $t < T$, we have $\|\theta_t - \theta_0\| \leq 1$ and
\begin{equation*}
- \nabla \mathcal L_{\text{CE}}(\theta_t)^\top \vv/\|\nabla \mathcal L_{\text{CE}}(\theta_t)\| \geq 1/6\,.
\end{equation*}
\end{proposition}
\begin{proof}
We will prove the claim by induction. 
We write
\begin{equation*}
- \nabla \mathcal L_{\text{CE}}(\theta_t) = \frac{1}{2n}\sum_{i=1}^n \pmlabels^{[i]} \vx^{[i]} - \frac{1}{2n}\sum_{i=1}^n \vx^{[i]} \tanh(\theta_t^\top \vx^{[i]})\,.
\end{equation*}
Since $\E \vv^\top(\pmlabels^{[i]} \vx^{[i]}) = (1- \Delta)$, the law of large numbers implies
\begin{equation*}
\vv^\top \Big(\frac{1}{2n}\sum_{i=1}^n \pmlabels^{[i]} \vx^{[i]}\Big) = \frac 12 (1 - \Delta) + o_P(1)\,.
\end{equation*}
Moreover, by Lemma \ref{lem:sup_bound}, there exists a positive constant $c$ such that with high probability
\begin{align*}
\Big|\vv^\top\Big(\frac{1}{2n}\sum_{i=1}^n \vx^{[i]} \tanh(\theta_t^\top \vx^{[i]})\Big)\Big| &\leq \frac{1}{2} \Big(\frac{1}{n} \sum_{i=1}^n (\vv^\top \vx^{[i]})^2 \Big)^{1/2}\left(\frac{1}{n} \sum_{i=1}^n \tanh(\theta_t^\top \vx^{[i]})^2\right)^{1/2} \\
& \leq \frac 12 |\tanh(\theta_t^\top \vv)| + c\sigma(1 + \|\theta_t - \theta_0\|)\,.
\end{align*}
Thus, applying Lemma~\ref{lem:norm} yields that with high probability
\begin{equation}\label{eq:recursive_bound}
- \nabla \mathcal L_{\text{CE}}(\theta_t)^\top v/\|\nabla \mathcal L_{\text{CE}}(\theta_t)\| \geq \frac 12 ( (1 - \Delta) - |\tanh(\theta_t^\top \vv)|) - c\sigma(1 + \|\theta_t - \theta_0\|)\,.
\end{equation}
When $t = 0$, the first term is $\frac 12 (1-\Delta) + o_P(1)$ by Lemma~\ref{lem:init}, and the second term is $c \sigma$.
Since we have assumed that $\Delta \leq 1/2$, as long as $\sigma \leq \sigma_\Delta < c^{-1} \left(\frac 23 - \frac \Delta 2\right)$ we will have that $- \nabla \mathcal L_{\text{CE}}(\theta_0)^\top v/\|\nabla \mathcal L_{\text{CE}}(\theta_0)\| \geq 1/6$ with high probability, as desired.

We proceed with the induction.
We will show that $\|\theta_t - \theta_0\| \leq 1$ with high probability for $t < T$, and use~\eqref{eq:recursive_bound} to show that this implies the desired bound on the correlation of the gradient.
If we assume the claim holds up to time $t$, then the definition of gradient descent implies
\begin{equation*}
\theta_t - \theta_0 = \eta\sum_{s = 0}^{t-1} \mathbf g_s\,,
\end{equation*}
where $\mathbf g_s$ satisfies $\mathbf g_s^\top \vv/\|\mathbf g_s\| \geq 1/6$.
Since the set of vectors satisfying this requirement forms a convex cone, we obtain that
\begin{equation*}
(\theta_t - \theta_0)^\top \vv/\|\theta_t - \theta_0\| \geq 1/6\,
\end{equation*}
From this observation, we obtain two facts about $\theta_t$.
First, since $t < T$, the definition of $T$ implies that $\theta_t^\top \vv < .1$.
Since $|\theta_0^\top \vv| = o_P(1)$ by Lemma~\ref{lem:init}, we obtain that $\|\theta_t - \theta_0\| \leq 1$ with high probability.
Second, $\theta_t^\top \vv \geq \theta_0^\top \vv$, and since $|\theta_0^\top \vv| = o_P(1)$ we have in particular that $\theta_t^\top \vv > -.1$.
Therefore, with high probability, we also have $|\theta_t \top \vv| < .1$

Examining~\eqref{eq:recursive_bound}, we therefore see that the quantity on the right side is at least
\begin{equation*}
    \frac 12 ( (1 - \Delta) - .1) - 2 c \sigma\,.
\end{equation*}
Again since we have assumed that $\Delta \leq 1/2$, as long as $\sigma \leq \sigma_\Delta < (2c)^{-1} \left(\frac 23 - \frac \Delta 2 - .1\right)$, we obtain by~\eqref{eq:recursive_bound} that $- \nabla \mathcal L_{\text{CE}}(\theta_t)^\top v/\|\nabla \mathcal L_{\text{CE}}(\theta_t)\| \geq 1/6$.

\end{proof}

Given $\theta_t$, we denote by
\begin{equation*}
\hat{\mathcal A}(\theta_t) = \frac{1}{|\wrong|} \sum_{i \in \wrong} \1\{\mathrm{sign}(\theta_t^\top \vx^{[i]}) = (\pmlabelstrue)^{[i]}\}
\end{equation*}
the accuracy of $\theta_t$ on mislabeled examples.
We now show that the classifier's accuracy on the mislabeled examples improves over the first $T$ rounds.
In fact, we show that with high probability, $\hat{\mathcal A}(\theta_0) \approx 1/2$ whereas $\hat{\mathcal A}(\theta_T) \approx 1$.
\begin{theorem}
There exists a $\sigma_\Delta$ such that if $\sigma \leq \sigma_\Delta$, then
\begin{equation}
    \hat{\mathcal A}(\theta_0) \leq .5001
\end{equation}
and
\begin{equation}
    \hat{\mathcal A}(\theta_T) > .9999
\end{equation}
with high probability.
\end{theorem}
\begin{proof}
Let us write $\vx^{[i]} = (\pmlabelstrue)^{[i]} (\vv - \sigma \vz^{[i]})$, where $\vz^{[i]}$ is a standard Gaussian vector.
If we fix $\theta_0$, then $\mathrm{sign}(\theta_0^\top \vx^{[i]}) = (\pmlabelstrue)^{[i]}$ if and only if $\sigma \theta_0^\top \vz^{[i]} < \theta_0^\top \vv$.
In particular this yields
\begin{equation*}
\E[\1\{\mathrm{sign}(\theta_0^\top \vx^{[i]}) = (\pmlabelstrue)^{[i]}\}| \theta_0] = \PP[\sigma \theta_0^\top \vz^{[i]} < \theta_0^\top \vv | \theta_0] \leq 1/2 + O(|\theta_0^\top \vv|/\sigma)\,.
\end{equation*}
By the law of large numbers, we have that, conditioned on $\theta_0$,
\begin{equation*}
\hat{\mathcal A}(\theta_0) \leq 1/2 + O(|\theta_0^\top \vv|/\sigma) + o_P(1)\,,
\end{equation*}
and applying Lemma~\ref{lem:init} yields $\hat{\mathcal A}(\theta_0) \leq 1/2 + o_P(1)$.

In the other direction, we employ a method based on~\citep{mendelson2014learning}.
The proof of Proposition~\ref{prop:correlation} establishes that $\|\theta_t - \theta_0\| \leq 1$ for all $t < T-1$ with high probability. Since $\eta < 1$ and $\|\theta_0\| = 2$, Lemma~\ref{lem:norm} implies that as long as $\sigma < 1/2$, $\|\theta_T\| \leq 5$ with high probability.
Since $\theta_T^\top \vv \geq .1$ by assumption, we obtain that $\theta_T^\top \vv/\|\theta_T\| \geq 1/50$ with high probability.

Note that $W$ is a random subset of $[n]$. For now, let us condition on this random variable.
If we write $\Phi$ for the Gaussian CDF, then by the same reasoning as above, for any fixed $\theta \in \R^p$,
\begin{equation*}
    \E[\1\{\mathrm{sign}(\theta^\top \vx^{[i]}) = (\pmlabelstrue)^{[i]}\}] = \PP[\sigma \theta^\top \vz^{[i]} < \theta^\top \vv] = \Phi(\sigma^{-1} \theta^\top \vv/\|\theta\|)
\end{equation*}
Therefore, if $\theta^\top \vv/\|\theta\| \geq \tau$, then for any $\delta > 0$, we have
\begin{equation}\label{eq:small_ball_1}
    \hat{\mathcal{A}}(\theta) \geq \Phi(\sigma^{-1}\tau - \delta) - \frac{1}{|\wrong|} \sum_{i \in \wrong} \Phi(\sigma^{-1}\tau - \delta) - \1\{\theta^\top \vz^{[i]}/\|\theta\| < \sigma^{-1} \tau\}
\end{equation}
Set
\begin{equation*}
    \phi(x) := \left\{\begin{array}{ll}
    1 & \text{if $x < \sigma^{-1} \tau - \delta$} \\
    \frac 1 \delta (\sigma^{-1} \tau - x) & \text{if $x \in[\sigma^{-1} \tau - \delta, \sigma^{-1} \tau]$} \\
    0 & \text{if $x > \sigma^{-1} \tau$.}
    \end{array}\right.
\end{equation*}
By construction, $\phi$ is $\frac 1 \delta$-Lipschitz and satisfies
\begin{equation*}
    \1\{x < \sigma^{-1} \tau - \delta\} \leq \phi(x) \leq \1\{x < \sigma^{-1} \tau\}
\end{equation*}
for all $x \in \R$.
In particular, we have
\begin{equation*}
    \Phi(\sigma^{-1}\tau - \delta) - \1\{\theta^\top \vz^{[i]}/\|\theta\| < \sigma^{-1} \tau\} \leq \E [\phi(\theta^\top \vz^{[i]}/\|\theta\|)] - \phi(\theta^\top \vz^{[i]}/\|\theta\|)\,.
\end{equation*}
Denote the set of $\theta \in \R^p$ satisfying $\theta^\top \vv/\|\theta\| \geq \tau$ by $\mathcal{C}_\tau$.
Combining the last display with~\eqref{eq:small_ball_1} yields
\begin{equation*}
    \E \inf_{\theta \in \mathcal{C}_\tau} \hat{\mathcal{A}}(\theta) \geq \Phi(\sigma^{-1}\tau - \delta) - \E \sup_{\theta \in \mathcal{C}_\tau} \frac{1}{|\wrong|} \sum_{i \in \wrong} \E [\phi(\theta^\top \vz^{[i]}/\|\theta\|)] - \phi(\theta^\top \vz^{[i]}/\|\theta\|)\,.
\end{equation*}
To control the last term, we employ symmetrization and contraction (see \citep[Chapter 4]{ledoux2013probability}) to obtain
\begin{align*}
    \E \sup_{\theta \in \mathcal{C}_\tau} \frac{1}{|\wrong|} \sum_{i \in \wrong} \E [\phi(\theta^\top \vz^{[i]}/\|\theta\|)] - \phi(\theta^\top \vz^{[i]}/\|\theta\|) & \leq \E \sup_{\theta \in \mathcal{C}_\tau} \frac{1}{|\wrong|} \sum_{i \in \wrong} \epsilon_i \phi(\theta^\top \vz^{[i]}/\|\theta\|) \\
    & \leq \frac 1 \delta \E \sup_{\theta \in \mathcal{C}_\tau} \frac{1}{|\wrong|} \sum_{i \in \wrong} \epsilon_i \theta^\top \vz^{[i]}/\|\theta\| \\
    & \leq \frac 1 \delta \E \sup_{\theta \in \R^p} \frac{1}{|\wrong|} \sum_{i \in \wrong} \epsilon_i \theta^\top \vz^{[i]}/\|\theta\| \\
    & = \frac 1 \delta \E \left\|\frac{1}{|\wrong|} \sum_{i \in \wrong} \epsilon_i \vz^{[i]}\right\|\,.
\end{align*}
where $\epsilon_i$ are independent Rademacher random variables.
The final quantity is easily seen to be at most $\frac 1 \delta \sqrt{p/|\wrong|}$.
Therefore we have
\begin{equation*}
    \E \inf_{\theta \in \mathcal C_\tau} \hat{\mathcal A}(\theta) \geq \Phi(\sigma^{-1} \tau - \delta) - \frac 1 \delta \sqrt{p/|\wrong|}\,,
\end{equation*}
and a standard application of Azuma's inequality implies that this bound also holds with high probability.
Since $\theta_T^\top \vv/\|\theta_T\| \geq 1/50$ and $|\wrong| \geq \Delta n/2$ with high probability, there exists a positive constant $c_\Delta$ such that
\begin{equation*}
    \hat{\mathcal A}(\theta_T) \geq \Phi((50\sigma)^{-1} - \delta) - c_\Delta/\delta\,.
\end{equation*}
If we choose $\delta= 10^{-4}/2c_\Delta$, then there exists a $\sigma_\Delta$ for which $\Phi((50\sigma_\Delta)^{-1} - \delta) > 1 - 10^{-4}/2$.
We obtain that for any $\sigma \leq \sigma_\Delta$, $ \hat{\mathcal A}(\theta_T) \leq 1 - 10^{-4}$, as claimed.
\end{proof}

\subsection{Vanishing gradients}\label{sec:gradients}
We now show that, over the first $T$ iterations, the coefficients $\tanh(\theta^\top \vx^{[i]}) - \pmlabels^{[i]}$ associated with the correctly labeled examples decrease, while the coefficients on mislabeled examples increase.
For simplicity, we write $\coeff^{[i]} := \tanh(\theta^\top \vx^{[i]}) - \pmlabels^{[i]}$.
\begin{proposition}
There exists a constant $\sigma_\Delta$ such that, for any $\sigma \leq \sigma_\Delta$, with high probability,
\begin{align*}
\frac{1}{|\correct|}\sum_{i \in \correct} (\coeff^{[i]}(\theta_T))^2 &< \frac{1}{|\correct|} \sum_{i \in \correct} (\coeff^{[i]}(\theta_0))^2 - .05\\
\frac{1}{|W|}\sum_{i \in W} (\coeff^{[i]}(\theta_T))^2 &> \frac{1}{|W|}\sum_{i \in W} (\coeff^{[i]}(\theta_0))^2 + .05 \,.
\end{align*}
That is, during the first stage, the coefficients on correct examples decrease while the coefficients on wrongly labeled examples increase.
\end{proposition}
\begin{proof}
Let us first consider
\begin{equation*}
\frac{1}{|\correct|} \sum_{i \in \correct} (\tanh(\theta_0^\top \vx^{[i]}) - \pmlabels^{[i]})^2
\end{equation*}
For fixed initialization $\theta_0$, the law of large numbers implies that this quantity is near
\begin{equation*}
\E_{\vx, \pmlabels} (\pmlabels\tanh(\theta_0^\top \vx) - 1)^2 \geq \left(\E_{\vx, \pmlabels}\pmlabels\tanh(\theta_0^\top \vx) - 1\right)^2\,.
\end{equation*}
Let us write $\vx = \pmlabelstrue (\vv - \sigma \vz)$, where $\vz$ is a standard Gaussian vector.
Then the fact that $\tanh$ is Lipschitz implies
\begin{equation*}
\E_{\vx, \pmlabels}\pmlabels\tanh(\theta_0^\top \vx) \leq \E_{\vx, \pmlabels}\pmlabels\tanh(\pmlabelstrue \sigma \theta_0^\top \vz) + |\theta_0^\top \vv| = |\theta_0^\top \vv|\,,
\end{equation*}
where we have used that $\E[\tanh(\pmlabelstrue \sigma\theta_0^\top \vz)| \pmlabels] = 0$.
By Lemma~\ref{lem:init}, $|\theta_0^\top \vv| = o_P(1)$.
Hence
\begin{equation*}
\frac{1}{|\correct|} \sum_{i \in \correct} (\tanh(\theta_0^\top \vx^{[i]}) - \pmlabels^{[i]})^2 \geq 1 - o_P(1)\,.
\end{equation*}

At iteration $T$, we have that $\theta_T^\top \vv \geq .1$, by assumption, and $\|\theta_T - \theta_0\| \leq 3$, by the proof of Proposition~\ref{prop:correlation}.
We can therefore apply Lemma~\ref{lem:sup_bound} to obtain
\begin{align*}
    \left(\frac{1}{|C|} \sum_{i \in C} (\coeff^{[i]})^2 \right)^{1/2} & \leq \left(\frac{1}{|C|} \sum_{i \in C} ((\pmlabelstrue)^{[i]} \tanh(\theta_T^\top \vv) - \pmlabels^{[i]}) \right)^{1/2} + \sigma(2 + 3 c_\Delta) + o_P(1) \\
    &= |\tanh(\theta_T^\top \vv) - 1| + + \sigma(2 + 3 c_\Delta) + o_P(1) \\\
    & \leq |\tanh(.1) - 1| + \sigma(2 + 3 c_\Delta) + o_P(1)\,,
\end{align*}
where the equality uses the fact that $(\pmlabelstrue)^{[i]} = \pmlabels^{[i]}$ for all $i \in C$.
As long as $\sigma \leq \sigma_\Delta < .01/(2 + 3 c_\Delta)$ this quantity is strictly less than $.95$.
We therefore obtain that, for $\sigma \leq \sigma_\Delta$, $\frac{1}{|\correct|}\sum_{i \in \correct} (\coeff^{[i]}(\theta_T))^2 < \frac{1}{|\correct|} \sum_{i \in \correct} (\coeff^{[i]}(\theta_0))^2 - .05$ with high probability. This proves the first claim.

The second claim is established by an analogous argument: for fixed initialization $\theta_0$, we have
\begin{equation*}
    \E \tanh(\theta_0^\top \vx)^2 \leq \E (\theta_0^\top \vx)^2 = 4 \sigma^2 + (\theta_0^\top \vv)^2\,,
\end{equation*}
so as above we can conclude that
\begin{equation*}
\frac{1}{|\wrong|} \sum_{i \in \wrong} (\coeff^{[i]}(\theta_0))^2 \leq 1 + 4\sigma^2 + o_P(1)\,.
\end{equation*}
We likewise have by another application of Lemma~\ref{lem:sup_bound}
\begin{align*}
\left(\frac{1}{|\wrong|} \sum_{i \in \wrong} (\tanh(\theta_T^\top \vx^{[i]}) - \pmlabels^{[i]})^2\right)^{1/2} & \geq |\tanh(-\theta_T^\top \vv) - 1| - \sigma(2 + 3 c_\Delta) - o_P(1) \\
& \geq 1 + \tanh(.1) - \sigma(2 + 3 c_\Delta) - o_P(1)\,.
\end{align*}
where we again have used that $(\pmlabelstrue)^{[i]} = - \pmlabels^{[i]}$ for $i \in W$.
If we choose $\sigma_\Delta$ small enough that
\begin{equation*}
    1.05 + 4 \sigma_\Delta^2 < (1 + \tanh(.1) - \sigma_\Delta(2 + 3 c_\Delta))^2\,,
\end{equation*}
then we will have for all $\sigma \leq \sigma_\Delta$ that  $\frac{1}{|W|}\sum_{i \in W} (\coeff^{[i]}(\theta_T))^2 > \frac{1}{|W|}\sum_{i \in W} (\coeff^{[i]}(\theta_0))^2 + .05$ with high probability.
This proves the claim.
\end{proof}

\subsection{Memorization}\label{sec:memorization}
To show that the labels are memorized asymptotically, it suffices to show that the classes $S_+ := \{\vx^{[i]}: \pmlabels^{[i]} = +1\}$ and $S_- := \{\vx^{[i]}: \pmlabels^{[i]} = -1\}$ are linearly separable.
Indeed, it is well known that for linearly separable data, gradient descent performed on the logistic loss will yield a classifier which perfectly memorizes the labels~\citep[see, e.g.][Lemma 1]{soudry2018implicit}.
It is therefore enough to establish the following theorem.
\begin{theorem}\label{thm:asymptotic_memorization_appendix}
If $p, n \to \infty$ and $\liminf_{p, n \to \infty} p/n > 1 - \Delta/2$, then the classes $S_+$ and $S_-$ are linearly separable with probability tending to $1$.
\end{theorem}
\begin{proof}
Write $X = \{\vx^{[1]}, \dots, \vx^{[n]}\}$.
Since the samples $\vx^{[1]}, \dots, \vx^{[n]}$ are drawn from a distribution absolutely continuous with respect to the Lebesgue measure, they are in general position with probability $1$.
By a theorem of Schl\"{a}fli \citep[see][Theorem 1]{cover65}, there exist
\begin{equation*}
    C(n, p) := 2 \sum_{k = 0}^{p-1}\binom{n-1}{k}
\end{equation*}
different subsets $S \subseteq X$ that are linearly separable from their complements.
In particular, there are at most
\begin{equation*}
    2^{n} - C(n, p) =  2 \sum_{k=p}^{n-1}\binom{n-1}{k} = 2 \sum_{k=0}^{n - p - 1} \binom{n-1}{k}
\end{equation*}
partitions of $X$ which are \emph{not} separable.

Write $B$ for the bad set of non-separable subsets $S \subseteq X$.
Conditional on $X$, the probability that the classes $S_+$ and $S_-$ are not separable is just $\PP[S_+ \in B | X]$.

Let us write $T_+:= \{\vx^{[i]}: (\pmlabelstrue)^{[i]} = +1\}$.
For each $i$, the example $\vx^{[i]}$ is in $S_+$ with probability $1 - (\Delta/2)$ if $i \in T_+$ or $\Delta/2$ or if $i \notin T_+$.
We therefore have for any $S \subset X$ that
\begin{equation*}
    \PP[S_+ = S | X] = (\Delta/2)^{|T_+ \triangle S|} (1-(\Delta/2))^{n - |T_+ \triangle S|}\,.
\end{equation*}
We obtain that
\begin{align}
    \PP[S_+ \in B| X] & = \sum_{S \in B} (\Delta/2)^{|T_+ \triangle S|} (1-(\Delta/2))^{n - |T_+ \triangle S|} \nonumber \\
    & = \sum_{k = 0}^n |\{S \in B : |T_+ \triangle S| = k\}| \cdot (\Delta/2)^{k} (1-(\Delta/2))^{n - k} \label{eq:sum}\,.
\end{align}
The set $\{S \in B : |T_+ \triangle S| = k\}$ has cardinality at most $\binom{n}{k}$.
Moreover, $\sum_{k = 0}^n |\{S \in B : |T_+ \triangle S| = k\}| = |B|$. We can therefore bound the sum \eqref{eq:sum} by the following optimization problem:
\begin{align}
    \max_{x_1, \dots x_n} \,\, & \sum_{k=0}^n x_k \cdot (\Delta/2)^{k} (1-(\Delta/2))^{n - k} \label{eq:prog}\\
    &\text{s.t.}\,\,\, x_k \in \left[0, \binom{n}{k}\right], \sum_{k=0}^n x_k = |B|\,.\nonumber
\end{align}
Since $\Delta \leq 1$, the probability $(\Delta/2)^k (1-(\Delta/2))^{n-k}$ is a nonincreasing function of $k$.
Therefore, because $|B| \leq 2 \sum_{k=0}^{n-p-1} \binom{n-1}{k} \leq 2 \sum_{k=0}^{n-p} \binom{n}{k}$, the value of~\eqref{eq:prog} is less than
\begin{equation*}
    2 \sum_{k=0}^{n-p} \binom{n}{k} (\Delta/2)^{k} (1-(\Delta/2))^{n - k} = 2 \cdot \PP[\mathrm{Bin}(n, \Delta/2) \leq (n-p)]\,.
\end{equation*}
If $\limsup_{n, p \to \infty} 1 - p/n < \Delta/2$, then this probability approaches $0$ by the law of large numbers.
We have shown that if $\liminf_{n, p} p/n > 1 - \Delta/2$, then
\begin{equation*}
    \PP[S_+ \in B | X] \leq 2 \cdot \PP[\mathrm{Bin}(n, \Delta/2) \leq (n-p)] = o(1)
\end{equation*}
holds $X$-almost surely, which proves the claim.

\end{proof}
\subsection{Additional lemmas}
\label{sec:lemmas}
\begin{lemma}\label{lem:init}
Suppose that $\theta_0$ is initialized randomly on the sphere of radius $2$.
\begin{equation*}
|\theta_0^\top \vv| = o_P(1).
\end{equation*}
\end{lemma}
\begin{proof}
Without loss of generality, take $\vv = \mathbf e_1$, the first elementary basis vector. Since the coordinates of $\theta_0$ each have the same marginal distribution and $\|\theta_0\|^2 = 2$ almost surely, we must have $\E |\theta_0^\top \mathbf e_1|^2 = 2/p$.
The claim follows.
\end{proof}

\begin{lemma}\label{lem:norm}
\begin{equation*}
\sup_{\theta \in \R^d} \|\nabla \mathcal L_{\text{CE}}(\theta)\| \leq 1 + 2\sigma + o_P(1)\,.
\end{equation*}
\end{lemma}
\begin{proof}
Denote by $\alpha$ the vector with entries $\alpha_i = \frac{1}{2 \sqrt n} [\tanh(\theta^\top \vx^{[i]}) - \pmlabels^{[i]}]$.
Since $|\tanh(x)| \leq 1$ for all $x \in \R$, we have $\|\alpha\| \leq 1$.
Therefore
\begin{equation*}
\nabla \mathcal L_{\text{CE}}(\theta) = \frac{1}{\sqrt n} \sum_{i=1}^n \vx^{[i]}  \alpha_i \leq \left\|\frac{1}{\sqrt n}\mathbf X\right\|\,,
\end{equation*}
where $\mathbf X \in \R^{p \times n}$ is a matrix whose columns are given by the vectors $\vx^{[i]}$.
By Lemma~\ref{lem:concentration}, we have
\begin{equation*}
\left\|\frac{1}{\sqrt n}\mathbf X\right\| = \left\|\frac{1}{n}\mathbf X \mathbf X^\top \right\|^{1/2}\leq 1 + 2\sigma + o_P(1)\,.
\end{equation*}
This yields the claim.
\end{proof}

\begin{lemma}\label{lem:sup_bound}
Fix an initialization $\theta_0$ satisfying $\|\theta_0\| = 2$.
For any $\tau > 0$ and for $I = \correct$ or $I = \wrong$, we have
\begin{equation*}
\sup_{\theta: \|\theta - \theta_0\| \leq \tau} \left(\frac{1}{|I|} \sum_{i \in I} ((\pmlabelstrue)^{[i]}\tanh(\theta^\top \vx^{[i]}) - \tanh(\theta^\top \vv))^2\right)^{1/2} \leq \sigma(2 + c_\Delta \tau) + o_P(1)\,.
\end{equation*}
The same claim holds with $I = [n]$ with $c_\Delta$ replaced by $2$.
\end{lemma}
\begin{proof}
Let us write $\vx^{[i]} = (\pmlabelstrue)^{[i]} (\vv - \sigma \vz^{[i]})$, where $\vz^{[i]}$ is a standard Gaussian vector.
Since $\tanh$ is odd and 1-Lipschitz, we have
\begin{equation*}
|(\pmlabelstrue)^{[i]}\tanh(\theta^\top \vx^{[i]}) - \tanh(\theta^\top \vv)| = |\tanh(\theta^\top \vv - \theta^\top \sigma \vz^{[i]}) - \tanh(\theta^\top \vv)| \leq \sigma |\theta^\top \vz^{[i]}|\,.
\end{equation*}
We therefore obtain
\begin{align*}
\Big(\frac{1}{|I|} \sum_{i \in I} ((\pmlabelstrue)^{[i]}\tanh(\theta^\top \vx^{[i]}) - \tanh(\theta^\top \vv))^2\Big)^{1/2} & \leq \sigma \Big(\frac{1}{|I|} \sum_{i \in I} (\theta^\top \vz^{[i]})^2\Big)^{1/2} \\
& \leq \sigma \Big(\frac{1}{|I|} \sum_{i \in I} (\theta_0^\top \vz^{[i]})^2\Big)^{1/2} + \sigma \Big(\frac{1}{|I|} \sum_{i \in I} ((\theta-\theta_0)^\top \vz^{[i]})^2\Big)^{1/2} \\
& \leq \sigma \Big(\frac{1}{|I|} \sum_{i \in I} (\theta_0^\top \vz^{[i]})^2\Big)^{1/2} + \sigma \|\theta - \theta_0\| \left\|\frac{1}{|I|} \sum_{i \in I} \vz^{[i]}(\vz^{[i]})^\top\right\|\,.
\end{align*}
Taking a supremum over all $\theta$ such that $\|\theta - \theta_0\| \leq \tau$ and applying Lemma~\ref{lem:concentration} yields the claim.
\end{proof}

\begin{lemma}\label{lem:concentration}
Assume $p \leq n$.
There exists a positive constant $c_\Delta$ depending on $\Delta$ such that for $I = \correct$ or $I=\wrong$, 
\begin{align*}
\frac{1}{|I|} \sum_{i \in I} (\theta_0^\top \vz^{[i]})^2 & \leq 2 + o_P(1)\\
\left\|\frac{1}{|I|} \sum_{i \in I} \vz^{[i]}(\vz^{[i]})^\top\right\|^{1/2} & \leq c_\Delta + o_P(1) \\
\left\|\frac{1}{|I|} \sum_{i \in I} \vx^{[i]}(\vx^{[i]})^\top\right\|^{1/2} & \leq 1 + \sigma c_\Delta + o_P(1)\,.
\end{align*}
Moreover, the same claims hold with $I = [n]$, when $c_\Delta$ can be replaced by $2$.
\end{lemma}
\begin{proof}
The first claim follows immediately from the law of large numbers.
For the second two claims, we first consider the case where $I = [n]$.
Let us write $\mathbf Z$ for the matrix whose columns are given by the vectors $\vz^{[i]}$.
Then
\begin{equation*}
    \left\|\frac{1}{n} \sum_{i \in [n]} \vz^{[i]}(\vz^{[i]})^\top\right\|^{1/2} = \left\|\frac 1n \mathbf Z \mathbf Z^{\top}\right\|^{1/2} = \left\|\frac 1{\sqrt n} \mathbf Z\right\| \leq 1 + \sqrt{p/n} + o_P(1)\,,
\end{equation*}
where the last claim is a  consequences of standard bounds for the spectral norm of Gaussian random matrices~\citep[see, e.g.][]{vershynin}.
Since $p \leq n$ by assumption, the claimed bound follows.
When $I = C$ or $W$, the same argument applies, except that we condition on the set of indices in $I$, which yields that, conditioned on $I$,
\begin{equation*}
    \left\|\frac{1}{|I|} \sum_{i \in I} \vz^{[i]}(\vz^{[i]})^\top\right\|^{1/2} \leq 2\sqrt{n/|I|} + o_P(1)\,.
\end{equation*}
For any $\Delta$, the random variable $|I|$ concentrates around its expectation, which is $c_\Delta n$, for some constant $c_\Delta$.

Finally, to bound $\left\|\frac{1}{|I|} \sum_{i \in I} \vx^{[i]}(\vx^{[i]})^\top\right\|^{1/2}$, we again let $\mathbf X$ be a matrix whose columns are given by $\vx^{[i]}$.
Then we can write
\begin{equation*}
    \mathbf X = \vv (\pmlabelstrue)^{\top} + \sigma \mathbf Z\,,
\end{equation*}
where $\mathbf Z$ is a Gaussian matrix, as above.
Therefore
\begin{equation*}
    \left\|\frac{1}{n} \sum_{i \in [n]} \vx^{[i]}(\vx^{[i]})^\top\right\|^{1/2} = \left\|\frac 1{\sqrt n}\mathbf X \right\| \leq \frac 1{\sqrt n} \left\|\vv (\pmlabelstrue)^{\top}\right\| + \sigma \left\|\frac 1{\sqrt n} \mathbf Z\right\| \leq 1 + 2 \sigma + o_P(1)\,.
\end{equation*}
The extension to $I = C$ or $W$ is as above.
\end{proof}
\newpage

\begin{figure}[t]
    \begin{tabular}{>{\centering\arraybackslash}m{0.13\linewidth} >{\centering\arraybackslash}m{0.4\linewidth} >{\centering\arraybackslash}m{0.4\linewidth}}
    & {\small Clean labels} & {\small Wrong labels}\\
    {\small Cross Entropy}
    & \includegraphics[width=\linewidth]{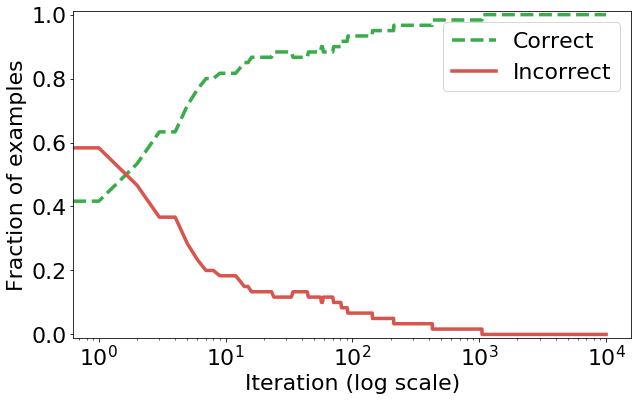}&
    \includegraphics[width=\linewidth]{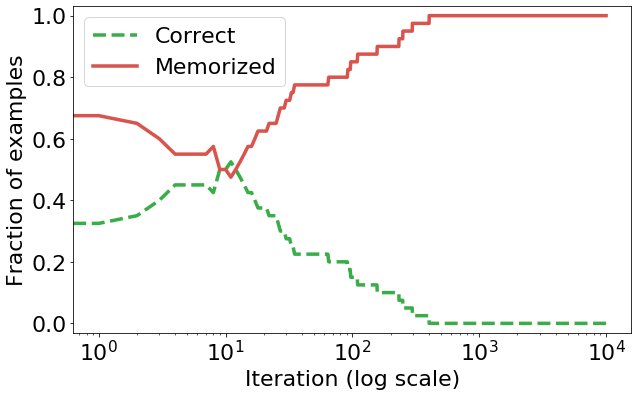}\\
    {\small \shortstack{Early-learning\\Regularization}}
    & \includegraphics[width=\linewidth]{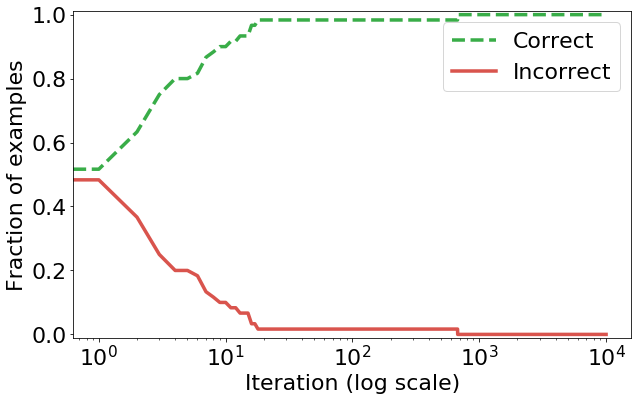}&
    \includegraphics[width=\linewidth]{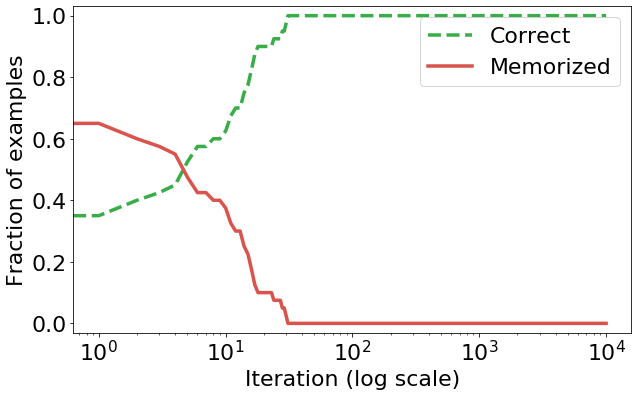}
  \end{tabular}

    \caption{Results of training a two-class softmax regression model with a traditional cross entropy loss (top row) and the proposed method (bottom row) to perform classification on 50 simulated data drawn from a mixture of two Gaussians in $\mathbb{R}^{100}$ with $\sigma=0.1$, where 40\% of the labels are flipped at random.  The plots show the fraction of examples with clean labels predicted correctly (green) and incorrectly (red) for examples with clean labels (left column) and wrong labels (right column). Analogously to the deep-learning model in Figure~\ref{fig:CE_vs_ELR}, the linear model trained with cross entropy begins by learning to predict the true labels, but eventually memorizes the examples with wrong labels. Early-learning regularization prevents memorization, allowing the model to continue learning on the examples with clean labels to attain high accuracy on examples with clean and wrong labels.}
    \label{fig:CE_vs_ELR_linear}
    
\end{figure}

\section{Early Learning and Memorization in Linear and Deep-Learning Models}
In this section we provide a numerical example to illustrate the theory in Section~\ref{sec:linear}, and the similarities between the behavior of linear and deep-learning models. We train the two-class softmax linear regression model described in Section~\ref{sec:linear} on data drawn from a mixture of two Gaussians in $\mathbb{R}^{100}$, where 40\% of the labels are flipped at random. Figure~\ref{fig:CE_vs_ELR_linear} shows the training accuracy on the training set for examples with clean and false labels. 
Analogously to the deep-learning model in Figure~\ref{fig:CE_vs_ELR}, the linear model trained with cross entropy begins by learning to predict the true labels, but eventually memorizes the examples with wrong labels as predicted by our theory. The figure also shows the results of applying our proposed early-learning regularization technique with temporal ensembling. ELR prevents memorization, allowing the model to continue learning on the examples with clean labels to attain high accuracy on examples with clean and wrong labels. 

As explained in Section~\ref{sec:ELR}, for both linear and deep-learning models the effect of label noise on the gradient of the cross-entropy loss for each example $i$ is restricted to the term $\vp^{[i]} - \vy^{[i]}$, where $\vp^{[i]}$ is the probability example assigned by the model to the example and $\vy^{[i]}$ is the corresponding label. Figure~\ref{fig:Gradient_CE} plots this quantity for the linear model described in the previous paragraph and for the deep-learning model from Figure~\ref{fig:CE_vs_ELR}. In both cases, the label noise flips the sign of the term on the wrong labels (left column). The magnitude of this term dominates after early learning (right column), eventually producing memorization of the wrong labels. 

\begin{figure}[t]
    \centering
    \begin{tabular}{>{\centering\arraybackslash}m{0.13\linewidth} >{\centering\arraybackslash}m{0.4\linewidth} >{\centering\arraybackslash}m{0.4\linewidth}}
    {\small Linear model}
    &  \includegraphics[width=\linewidth]{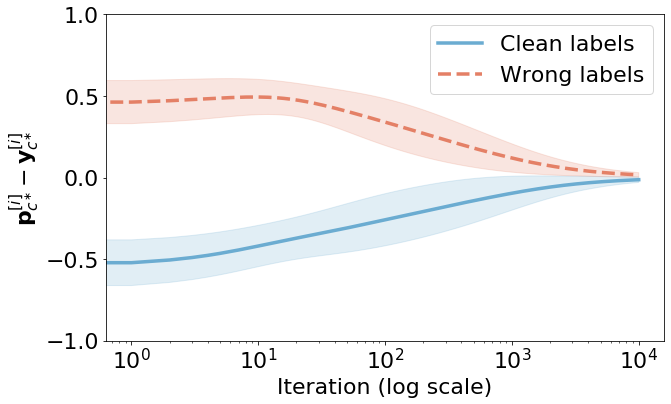}&
     \includegraphics[width=\linewidth]{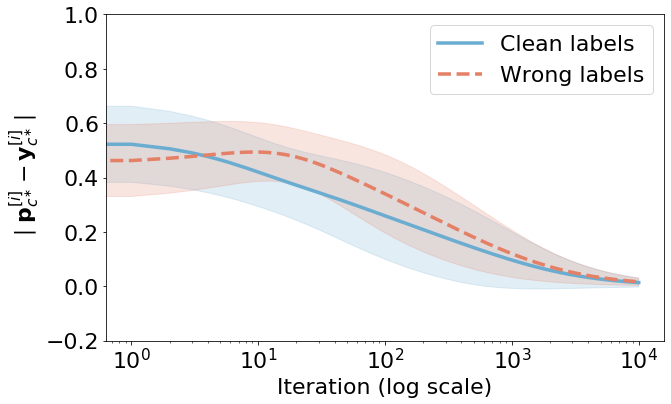}\\
    {\small Neural Network}
    &  \includegraphics[width=\linewidth]{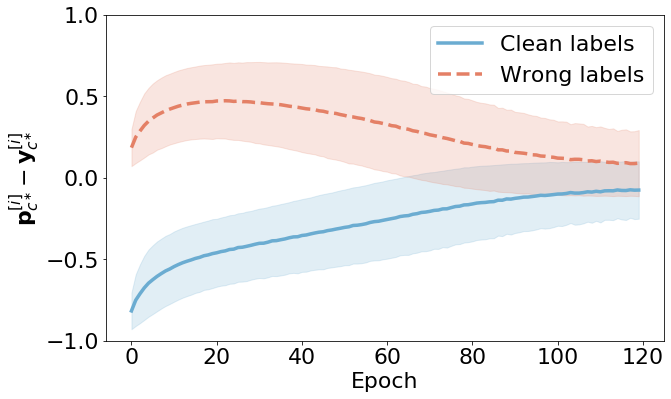}&
     \includegraphics[width=\linewidth]{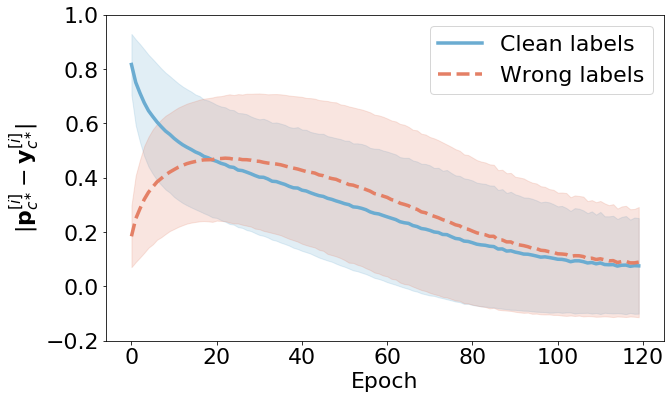}\\
  \end{tabular}
  
    \caption{The effect of label noise on the gradient of the cross-entropy loss for each example $i$ is restricted to the term $\vp^{[i]} - \vy^{[i]}$, where $\vp^{[i]}$ is the probability example assigned by the model to the example and $\vy^{[i]}$ is the corresponding label (see Section~\ref{sec:ELR}). The plots show this term (left column) and its magnitude (right column) for the same linear model as in Figure~\ref{fig:CE_vs_ELR_linear} (top row) and the same ResNet-34 on CIFAR-10 as in Figure~\ref{fig:CE_vs_ELR} (bottom row) with 40\% symmetric noise. On the left, we plot the entry of $\vp^{[i]} - \vy^{[i]}$ corresponding to the true class, denoted by $c^{\ast}$, for training examples with clean (blue) and wrong (red) labels. On the right, we plot the absolute value of the entry. During early learning, the clean labels dominate, but afterwards their effects decrease and the noisy labels start to be dominant, eventually leading to memorization of the wrong labels. In all plots the curves represent the mean value, and the shaded regions are within one standard deviation of the mean.}
    \label{fig:Gradient_CE}
\end{figure}
\section{Regularization Based on Kullback-Leibler Divergence \label{sec:kl}}
A natural alternative to our proposed regularization would be to penalize the Kullback-Leibler (KL) divergence between the the model output and the targets. This results in the following loss function
\begin{align} 
\mathcal{L}_\text{CE}(\param) -  \frac{ \lambda}{n}\sum_{i=1}^n\sum_{c=1}^{C} \vq^{[i]}_c \log \vp^{[i]}_c. \label{eq:kl_loss}
\end{align}
Figure~\ref{fig:KL_reg} shows the result of applying this regularization to CIFAR-10 dataset with 40\% symmetric noise for different values of the regularization parameter $\lambda$, using targets computed via temporal ensembling. In contrast to ELR, which succeeds in avoiding memorization while allowing the model to learn effectively as demonstrated in the bottom row of Figure~\ref{fig:CE_vs_ELR}, regularization based on KL divergence fails to provide robustness. When $\lambda$ is small, memorization of the wrong labels leads to overfitting as in cross-entropy minimization. Increasing $\lambda$ delays memorization, but does not eliminate it. Instead, the model starts overfitting the initial estimates, whether correct or incorrect, and then eventually memorizes the wrong labels (see the bottom right graph in Figure~\ref{fig:KL_reg}).

Analyzing the gradient of the cost function sheds some light on the reason for the failure of this type of regularization. The gradient with respect to the model parameters $\param$ equals
\begin{align}
    \frac{1}{n}\sum_{i=1}^n \nabla \func_{\vx^{[i]}}(\param) \left(\left(  \vp^{[i]} - \vy^{[i]}\right)+\lambda \left(  \vp^{[i]} - \vq^{[i]}\right) \right).
\end{align}
A key difference between this gradient and the gradient of ELR is the dependence of the sign of the regularization component on the targets. In ELR, the sign of the $c$th entry for the $i$th is determined by the difference between $\vq^{[i]}_c$ and the rest of the entries of $\vq^{[i]}$ (see Lemma~\ref{lemma:ELR_gradient}). In contrast, for KL divergence it depends on the difference between $\vq^{[i]}_c$ and $\vp^{[i]}_c$. This results in overfitting the target probabilities. To illustrate this, recall that for examples with clean labels, the cross-entropy term $\vp^{[i]} - \vy^{[i]}$ tends to vanish after the early-learning stage because $\vp^{[i]}$ is very close to $\vy^{[i]}$, allowing examples with wrong labels to dominate the gradient. Let $c^{\ast}$ denote the true class. When $\vp^{[i]}_{c^{\ast}}$ (correctly) approaches one, $\vq^{[i]}_{c^{\ast}}$ will generally tend to be smaller, because $\vq^{[i]}$ is obtained by a moving average and therefore tends to be smoother than $\vp^{[i]}$. Consequently, the regularization term tends to decrease $\vp^{[i]}_{c^{\ast}}$. This is exactly the opposite effect than desired. In contrast, ELR tends to keep $\vp^{[i]}_{c^{\ast}}$ large, as explained in Section~
\ref{sec:ELR}, which allows the model to continue learning on the clean examples.

\begin{figure}[t]
    \begin{tabular}{>{\centering\arraybackslash}m{0.1\linewidth} >{\centering\arraybackslash}m{0.4\linewidth} >{\centering\arraybackslash}m{0.4\linewidth}}
    $\lambda$ & {\small Clean labels} & {\small Wrong labels}\\
    {\small 1}
    & \includegraphics[width=\linewidth]{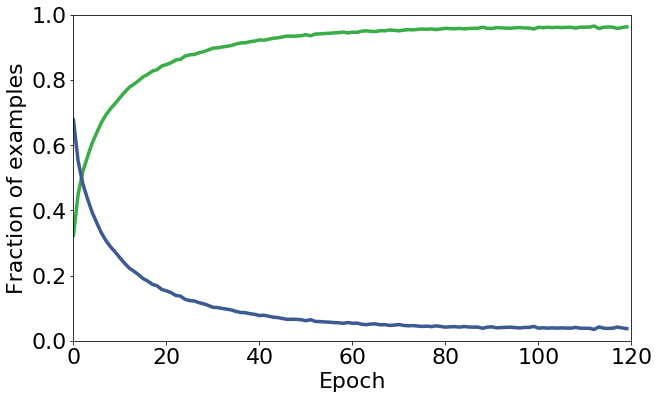}&
    \includegraphics[width=\linewidth]{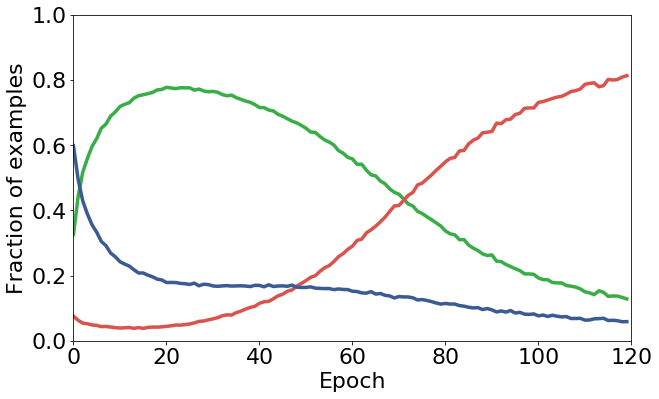}\\
    {\small 5}
    & \includegraphics[width=\linewidth]{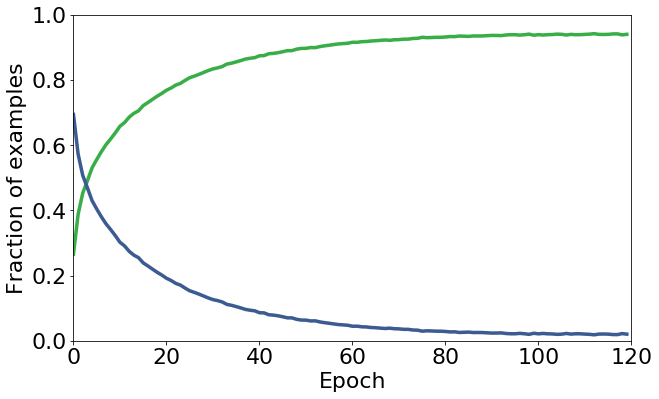}&
    \includegraphics[width=\linewidth]{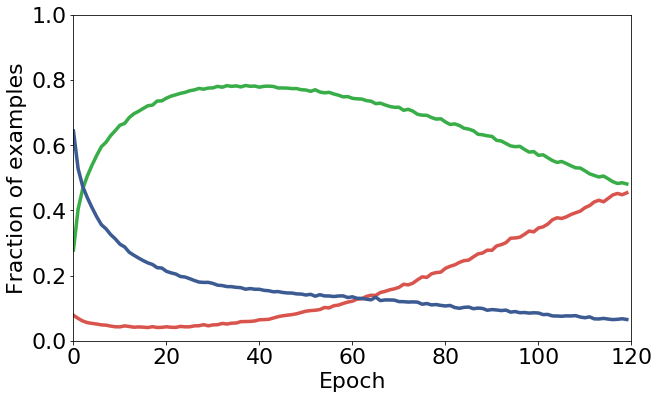}\\
    {\small 10}
    & \includegraphics[width=\linewidth]{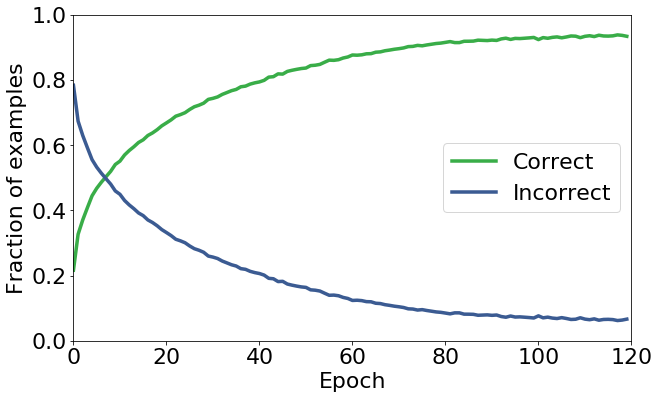}&
    \includegraphics[width=\linewidth]{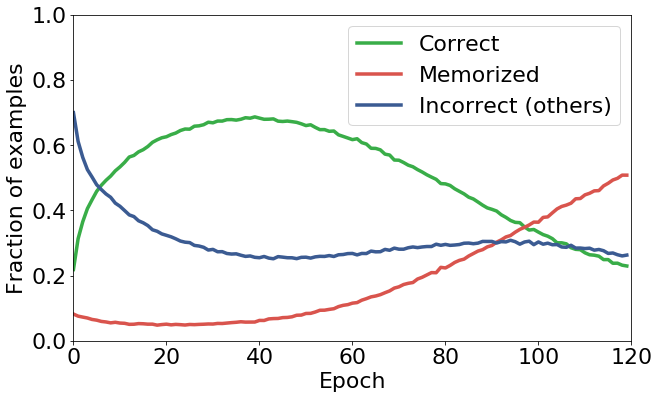}
  \end{tabular}

    \caption{
    Results of training a ResNet-34 neural network with a traditional cross entropy loss regularized by KL divergence using different coefficients $\lambda$ (showed in different rows) to perform classification on the CIFAR-10 dataset where 40\% of the labels are flipped at random.  The left column shows the fraction of examples with clean labels that are predicted correctly (green) and incorrectly (blue). The right column shows the fraction of examples with wrong labels that are predicted correctly (green), \emph{memorized} (the prediction equals the wrong label, shown in red), and incorrectly predicted as neither the true nor the labeled class (blue). When $\lambda = 1$, it is analogous to the model trained without any regularization (top row in Figure~\ref{fig:CE_vs_ELR}), while when $\lambda$ increases, the fraction of correctly predicted examples decreases, indicating worse performance.\vspace{-3mm}}
    \label{fig:KL_reg}
    
\end{figure}
\section{The Need for Early Learning Regularization}
\begin{figure}[t]
    \centering
    \includegraphics[width=0.5\linewidth]{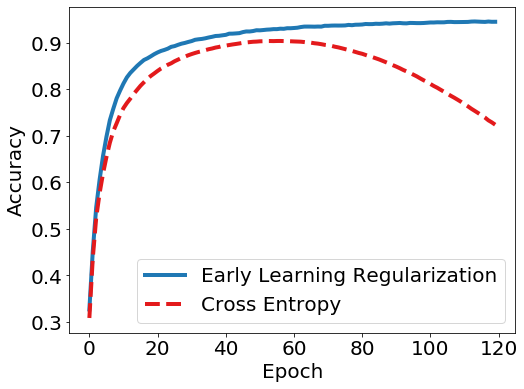}
    \caption{Validation accuracy achieved by targets estimated via temporal ensembling using the cross entropy loss and our proposed cost function. The model is a ResNet-34 trained on CIFAR-10 with 40\% symmetric noise. The temporal ensembling momentum $\beta$ is set to 0.7. Without the regularization term, the targets eventually overfit the noisy labels.}
    \label{fig:targets_training}
\end{figure}
Our proposed framework consists of two components: target estimation and the early-learning regularization term. Figure~\ref{fig:targets_training} shows that the regularization term is critical to avoid memorization. If we just perform target estimation via temporal ensembling while training with a cross-entropy loss, eventually the targets overfit the noisy labels, resulting in decreased accuracy.

\section{Proof of Lemma~\ref{lemma:ELR_gradient}\label{sec:proof_lemma}}
To ease notation, we ignore the $i$ superscript, setting $\vp :=\vp^{[i]}$ and $\vq:=\vq^{[i]}$. We denote the instance-level ELR by 
\begin{align} 
\mathcal{R}(\param) := \log \left(1-\langle \vp,\vq \rangle\right). 
\end{align}
The gradient of $\mathcal{R}$ is
\begin{align}
    \nabla \mathcal{R}(\param) & = \frac{1}{1-\langle \vp,\vq \rangle} \nabla \left(1-\langle \vp,\vq \rangle\right).
    \label{eq:reg_gradient}
\end{align}
We express the probability estimate in terms of the softmax function and the deep-learning mapping $\func_{\vx}{(\param)}$, $\vp := \frac{e^{\func_{\vx}(\param)}}{\sum_{c=1}^C e^{\left(\func_{\vx}(\param)\right)_c}}$, where $e^{\func_{\vx}(\param)}$ denotes a vector whose entries equal the exponential of the entries of $\func_{\vx}{(\param)}$. Plugging this into Eq.~(\ref{eq:reg_gradient}) yields 
\begin{align}
  \nabla \mathcal{R}(\param)  & = \sum_{i=1}^n\frac{1}{1-\langle \vp,\vq \rangle} \nabla \left(1-\frac{\langle e^{\func_{\vx}(\param)},\vq \rangle}{\sum_{c=1}^C e^{\left(\func_{\vx}(\param)\right)_c}}\right)\\
    &=\sum_{i=1}^n\frac{-1}{1-\langle \vp,\vq \rangle} \frac{\nabla\langle e^{\func_{\vx}(\param)},\vq \rangle \cdot \sum_{c=1}^C e^{\left(\func_{\vx}(\param)\right)_c} - \langle e^{\func_{\vx}(\param)},\vq \rangle\cdot \nabla  
    \sum_{c=1}^C e^{\left(\func_{\vx}(\param)\right)_c}
    }{\left(\sum_{c=1}^C e^{\left(\func_{\vx}(\param)\right)_c}\right)^2}\\
    & =\sum_{i=1}^n\frac{-\nabla \func_{\vx}(\param)}{1-\langle \vp,\vq \rangle} 
    \frac{ e^{\func_{\vx}(\param)}\odot \vq \cdot \sum_{c=1}^C e^{\left(\func_{\vx}(\param)\right)_c} - \langle e^{\func_{\vx}(\param)},\vq \rangle\cdot 
    e^{\func_{\vx}(\param)}
    }{\left(\sum_{c=1}^C
    e^{\left(\func_{\vx}(\param)\right)_c}\right)^2}\\
    & = \sum_{i=1}^n\frac{-\nabla \func_{\vx}(\param)}{1-\langle \vp,\vq \rangle}\left(
    \frac{e^{\func_{\vx}(\param)}\odot \vq}{\sum_{c=1}^C e^{\left(\func_{\vx}(\param)\right)_c}} - \frac{\langle e^{\func_{\vx}(\param)},\vq \rangle}{\sum_{c=1}^C e^{\left(\func_{\vx}(\param)\right)_c}}\cdot  \frac{e^{\func_{\vx}(\param)}}{\sum_{c=1}^C e^{\left(\func_{\vx}(\param)\right)_c}}\right).
\end{align}
The formula can be simplified to 
\begin{align}
     \nabla \mathcal{R}(\param) & = \frac{-\nabla \func_{\vx}(\param)}{1-\langle \vp,\vq \rangle}\left(\vp\odot \vq - \langle\vp,\vq\rangle\cdot \vp \right)\\
    & =   \frac{\nabla \func_{\vx}(\param)}{1-\langle \vp,\vq \rangle} \begin{bmatrix} \vp_1 \cdot \left( \langle\vp,\vq\rangle - \vq_1\right)\\ \vdots \\ \vp_C\cdot \left( \langle\vp,\vq\rangle - \vq_C\right) \end{bmatrix}\\
    & =   \frac{\nabla \func_{\vx}(\param)}{1-\langle \vp,\vq \rangle} \begin{bmatrix} \vp_1 \cdot \sum_{k=1}^C\left(\vq_k-\vq_1\right)\vp_k\\ \vdots \\  \vp_C \cdot \sum_{k=1}^C\left(\vq_k-\vq_C\right)\vp_k \end{bmatrix}.
\end{align}

\section{Algorithms}
\label{sec:algorithms}
\begin{algorithm}[t]
\caption{\label{alg:algo_ELR}\ \ 
Pseudocode for ELR with temporal ensembling. 
}
\begin{tabbing}
\Req $\{\vx^{[i]},\vy^{[i]}\},\; 1 \leq i \leq n$        \= = training data (with noisy labels) \hspace*{-16mm} \= \\
\Req $\beta$ \> = temporal ensembling momentum, $0 \le \beta < 1$ \\
\Req $\lambda$  \>  = regularization parameter \\
\Req $\func_\vx(\param)$\>    = neural network with trainable parameters $\param$\\
\X $\vq$    \= $\gets \mathbf{0}_{[n \times C]}$ \hspace*{7mm}                   \>   \cm{initialize ensemble predictions} \\
\X {\bf for} $t$ in $[1,\mathit{num\_epochs}]$ {\bf do} \\
\XX {\bf for} each minibatch $B$ {\bf do} \\
\XXX {\bf for} $i$ in $B$ {\bf do}\\
\XXXX $\vp^{[i]} \gets \smax\left(\func_{\vx_i}(\param)\right)$ 
        \>\> \cm{evaluate network outputs}\\
\XXXX $\vq^{[i]} \gets \beta\vq^{[i]} + (1-\beta)  \vp^{[i]}$                        \>\> \cm{temporal ensembling} \\
\XXX {\bf end for}\\
\XXX $\text{loss}{}\gets$\=${}-\frac{1}{|B|}\sum_{i=1}^{|B|}\sum_{c=1}^{C} \vy^{[i]}_c \log \smax\left(\func_{\vx_i}(\param)\right)_c$  \>   \cm{cross entropy loss component} \\
\XXX \>${}+ \frac{ \lambda}{|B|}\sum_{i\in B} \log \left(1-\langle \smax\left(\func_{\vx_i}(\param)\right),\vq^{[i]} \rangle\right) $                  \>   \cm{proposed regularization component} \\

\XXX update $\param$ using stochastic gradient descent                               \>\> \cm{update network parameters} \\
\XX {\bf end for} \\
\X {\bf end for}\\
\X {\bf return} $\param$
\end{tabbing}
\vspace*{-1.5mm}
\end{algorithm}

\begin{algorithm}[!h]
\caption{\label{alg:algo_ELR_plus}\ \ 
Pseudocode for ELR+. 
}
\begin{tabbing}
\Req $\{\vx^{[i]},\vy^{[i]}\},\; 1\leq i \leq n$        \= = training data (with noisy labels) \hspace*{-16mm} \= \\
\Req $\beta$ \> = temporal ensembling momentum, $0 \le \beta < 1$ \\
\Req $\gamma$\>  = weight averaging momentum, $0 \le \gamma < 1$\\
\Req $\lambda$  \>  = regularization parameter \\
\Req $\alpha$  \>  = mixup hyperparameter  \\

\Req $\func_\vx(\param_1)$\>    = neural network 1 with trainable parameters $\param_1$\\
\Req $\func_\vx(\param_2)$\>    = neural network 2 with trainable parameters $\param_2$\\
\X $\vq_1$, $\vq_2$   \= $\gets \mathbf{0}_{[n \times C]}$, $\mathbf{0}_{[n \times C]}$ \hspace*{5mm}                   \>   \cm{initialize averaged predictions} \\
\X $\bar{\param}_1$, $\bar{\param}_2$    \= $\gets \mathbf{0}$, $\mathbf{0}$ \hspace*{5mm}                   \>   \cm{initialize averaged weights (untrainable)} \\
\X {\bf for} $t$ in $[1,\mathit{num\_epochs}]$ {\bf do} \\
\XX {\bf for} $k$ in $[1,2]$ {\bf do}\>\>   \cm{for each network}\\
\XXX {\bf for} each minibatch $B$ {\bf do} \\
\XXXX $\tilde{B} \gets \text{mixup}(B, \alpha)$ \>\> \cm{\textit{mixup} augmentation on the mini-batch}\\
\XXXX $\bar{\param}_k = \gamma \bar{\param}_k + (1-\gamma)\param_k$  \>\> \cm{weight averaging}\\
\XXXX {\bf for} $i$ in $B$ {\bf do}\\
\XXXXX $\vp^{[i]} \gets 
\smax\left(\func_{\vx_{i}}(\bar{\param}_{\{1,2\}\setminus k})\right)$ 
        \>\> \cm{network evaluation with weight averaging}\\
\XXXXX $\vq_k^{[i]} \gets \beta \vq_k^{[i]} + (1-\beta)\vp^{[i]}$                     \>\> \cm{temporal ensembling} \\
\XXXX {\bf end for}\\
\XXXX $\text{loss}{}\gets$\=${}-\frac{1}{|B|}\sum_{i=1}^{|B|}\sum_{c=1}^{C} \vy^{[i]}_c \log\smax\left(\func_{\tilde{\vx}_i}(\param_k)\right)_c$  \>  \hspace*{3mm} \cm{cross entropy loss component} \\
\XX \>${}+\frac{ \lambda}{|B|}\sum_{i\in B} \log \left(1-\langle \smax\left(\func_{\tilde{\vx}_i}(\param_k),\tilde{\vq}^{[i]} \rangle\right)\right) $                  \>  \hspace*{3mm} \cm{proposed regularization component} \\

\XXXX update $\param_k$ using SGD                               \>\> \cm{update network parameters} \\
\XXX {\bf end for} \\
\XX {\bf end for}\\
\X {\bf end for}\\
\X {\bf return} $\param_1$, $\param_2$
\end{tabbing}
\vspace*{-1.5mm}
\end{algorithm}

Algorithm~\ref{alg:algo_ELR} and Algorithm~\ref{alg:algo_ELR_plus} provide detailed pseudocode for ELR combined with temporal ensembling (denoted simply by ELR) and ELR combined with temporal ensembling, weight averaging, two networks, and mixup data augmentation (denoted by ELR+)  respectively. For CIFAR-10 and CIFAR-100, we use the sigmoid shaped function $e^{-5(1-i/40000)^2}$ ($i$ is current training step, following~\citep{tarvainen2017mean}) to ramp-up the weight averaging momentum $\gamma$ to the value we set as a hyper-parameter. For the other datasets, we fixed $\gamma$. For CIFAR-100, we also use previously mentioned sigmoid shaped function to ramp up the coefficient $\lambda$ to the value we set as a hyper-parameter. Moreover, each entry of the labels $y$ will also be updated by the targets $t$ using $\frac{y_ct_c}{\sum_{c=1}^C y_ct_c}$ in CIFAR-100.

To apply \textit{mixup} data augmentation, when processing the $i$th example in a mini-batch $(\vx^{[i]},\vy^{[i]},\vq^{[i]})$, we randomly sample another example $(\vx^{[j]},\vy^{[j]},\vq^{[j]})$, and compute the $i$th mixed data $(\tilde{\vx}^{[i]},\tilde{\vy}^{[i]},\tilde{\vq}^{[i]})$ as follows: 
    \begin{align*}
        \ell &\sim \text{Beta}(\alpha,\alpha),\\
        \ell' &= \max (\ell, 1-\ell),\\
        \tilde{\vx}^{[i]} &= \ell' \vx^{[i]} + (1-\ell')\vx^{[j]},\\
        \tilde{\vy}^{[i]} &= \ell' {\vy}^{[i]} + (1-\ell'){\vy}^{[j]},\\
        \tilde{\vq}^{[i]} &= \ell' {\vq}^{[i]} + (1-\ell'){\vq}^{[j]},
    \end{align*}
where $\alpha$ is a fixed hyperparameter used to choose the symmetric beta distribution from which we sample the ratio of the convex combination between data points.

\section{Description of the Computational Experiments}
Source code for the experiments is available at \url{https://github.com/shengliu66/ELR}.
\label{sec:experiments_appendix}
\begin{table}[t]
\footnotesize
\begin{center}
\begin{tabular}{c|c|c|c|c|c}
\toprule
Data set & Train & Val & Test & Image size & \# classes \\
\midrule
\multicolumn{6}{c}{Datasets with Clean Annotation}\\
\midrule
CIFAR-10 & 45K & 5k& 10K & $32\times 32$ & 10 \\
CIFAR-100 & 45K & 5k & 10K & $32\times 32$ & 100 \\
\midrule
\multicolumn{6}{c}{Datasets with Real World Noisy Annotation}\\
\midrule
Clothing-1M & 1M & 14K &10K & $224\times 224$ & 14\\
Webvision1.0 & 66K & - &2.5K & $256\times 256$ & 50\\
\bottomrule
\end{tabular}
\end{center}
\caption{Description of the datasets used in our computational experiments, including the training, validation and test splits.}
\label{tab:data_describ}
\end{table}

\subsection{Dataset Information}
In our experiments we apply ELR and ELR+ to perform image classification on four benchmark datasets: CIFAR-10, CIFAR-100, Clothing-1M, and a subset of WebVision. Because CIFAR-10, CIFAR-100 do not have predefined validation sets, we retain 10\% of the training sets to perform validation. Table~\ref{tab:data_describ} provides a detailed description of each dataset.

\subsection{Data preprocessing}
We apply normalization and simple data augmentation techniques (random crop and horizontal flip) on the training sets of all datasets. The size of the random crop is set to be consistent with previous works~\cite{zhang2018generalized, Jiang2018MentorNetLD}: $32 \time 32 $ for the CIFAR datasets, $224\times 224$ for Clothing1M (after resizing to $256 \times 256$), and $227\times 227$ for WebVision.

\subsection{Training Procedure}
Below we describe the training procedure for ELR (i.e. the proposed approach with temporal ensembling) for the different datasets. The information for ELR+ is shown in Table~\ref{tab:ELR_plus_hyperparameters}. In ELR+ we ensemble the outputs of two networks during inference, as is customary for  methods that train two networks simultaneously~\cite{li2020dividemix,Han2018CoteachingRT}.

\textbf{CIFAR-10/CIFAR-100}: We use a ResNet-34~\cite{he2016deep} and train it using SGD with a momentum of 0.9, a weight decay of $0.001$, and a batch size of $128$. The network is trained for $120$ epochs for CIFAR-10 and $150$ epochs for CIFAR-100. We set the initial learning rate as 0.02, and reduce it by a factor of 100 after 40 and 80 epochs for CIFAR-10 and after 80 and 120 epochs for CIFAR-100. We also experiment with cosine annealing learning rate~\cite{Loshchilov2017SGDRSG} where the maximum number of epoch for each period is set to $10$, the maximum and minimum learning rate is set to $0.02$ and $0.001$ respectively, total epoch is set to 150. 

\textbf{Clothing-1M}: We use a ResNet-50 pretrained on ImageNet same as Refs.~\cite{Wang2019SymmetricCE,xiao2015learning}. The model is trained with batch size 64 and initial learning rate 0.001, which is reduced by $1/100$ after 5 epochs (10 epochs in total). The optimization is done using SGD with a momentum 0.9, and weight decay $0.001$. For each epoch, we sample 2000 mini-batches from the training data ensuring that the classes of the noisy labels are balanced.

\textbf{WebVision}: Following Refs.~\cite{Jiang2018MentorNetLD, li2020dividemix}, we use an InceptionResNetV2 as the backbone architecture. All other optimization details are the same as for CIFAR-10, except for the weight decay ($0.0005$) and the batch size ($32$).

\subsection{Hyperparameters selection\label{sec:hyperparameters_select}}
 We perform hyperparameter tuning on the CIFAR datasets via grid search: the temporal ensembling parameter $\beta$ is chosen from $\{0.5,0.7,0.9,0.99\}$ and the regularization coefficient $\lambda$ is chosen from $\{1, 3 , 5, 7, 10\}$ using the validation set. The selected values are $\beta = 0.7$ and $\lambda = 3$ for symmetric noise, $\beta = 0.9$ and $\lambda = 1$ for assymetric noise on CIFAR-10, and $\beta = 0.9$ and $\lambda = 7$ CIFAR-100. For Clothing1M and WebVision we use the same values as for CIFAR-10. As shown in Section~\ref{sec:hyperparameters}, the performance of the proposed method seems to be robust to changes in the hyperparameters. For ELR+, we use the same values for $\lambda$ and $\beta$. The $\textit{mixup}$ $\alpha$ is set to 1 (chosen from $\{0.1,2,5\}$ via grid search on the validation set) and the value of the weight averaging parameter $\gamma$ is set to $0.997$ (which is the default value in the public code of Ref.~\cite{tarvainen2017mean}) except Clothing1M, which is set to 0.9999.

\begin{table}[t]
\footnotesize
\begin{center}
\resizebox{\linewidth}{!}{
\begin{tabular}{c|cccc}
\toprule
 & CIFAR-10 & CIFAR-100 & Clothing-1M & Webvision\\
\midrule

batch size & 128 & 128 & 64 & 32\\
architecture & PreActResNet-18 & PreActResNet-18 & ResNet-50 (pretrained) & InceptionResNetV2\\
training epochs & 200 & 250 & 15 & 100\\
learning rate (lr) & 0.02 & 0.02&  0.002 & 0.02\\
lr scheduler & divide 10 at 150th epoch & divide 10 at 200th epoch & divide 10 at 7th epoch & divide 10 at 50th epoch\\
weight decay & 5e-4  & 5e-4 & 1e-3 & 5e-4\\
\bottomrule
\end{tabular}}
\end{center}
\label{tab:ELR_plus_hyperparameters}
\caption{Training hyperparameters for ELR+ on CIFAR-10, Clothing-1M and Webvision.}
\end{table}

\section{Sensitivity to Hyperparameters \label{sec:hyperparameters}}
The main hyperparameters of ELR are the temporal ensembling parameter $\beta$ and regularization coefficient $\lambda$. As shown in the left image of Figure~\ref{fig:ablation_hyperparameters}, performance is robust to the value of $\beta$, although it is worth noting that this is only as long as the momentum of the moving average is large. The performance degrades to 38\% when the model outputs are used to estimate the target without averaging (i.e. $\beta = 0$). The regularization parameter $\lambda$ needs to be large enough to neutralize the gradients of the falsely labeled examples but also cannot be too large, to avoid neglecting the cross entropy term in the loss. As shown in the center image of Figure~\ref{fig:ablation_hyperparameters}, the sensitivity to  $\lambda$ is also quite mild. Finally, the right image of Figure~\ref{fig:ablation_hyperparameters} shows results for ELR combined with mixup data augmentation for different values of the mixup parameter $\alpha$. Performance is again quite robust, unless the parameter becomes very large, resulting in a peaked distribution that produces too much mixing.

\begin{figure}[t]
\resizebox{\linewidth}{!}{
    \begin{tabular}{>{\centering\arraybackslash}m{0.33\linewidth} >{\centering\arraybackslash}m{0.33\linewidth} >{\centering\arraybackslash}m{0.33\linewidth}}
    {\small \; Temporal ensembling momentum $\beta$} & {\small \quad Regularization coefficient $\lambda$} & {\small mixup parameter $\alpha$}\\
    \includegraphics[width=1.1\linewidth]{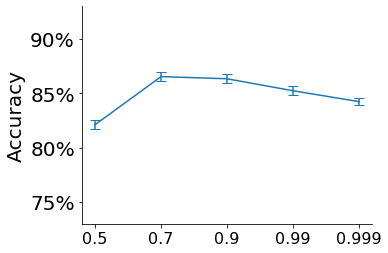} &
     \includegraphics[width=\linewidth]{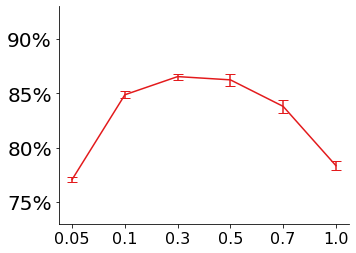} &{\vspace{3.5mm}\includegraphics[width=\linewidth]{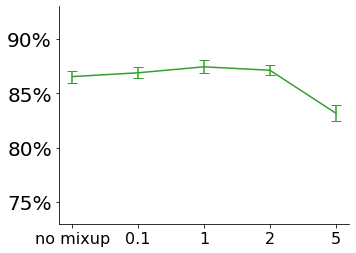}}

  \end{tabular}
}
    \caption{Test accuracy on CIFAR-10 with symmetric noise level 60\%. The mean accuracy over four runs is reported, along with bars representing one standard deviation from the mean. In each experiment, the rest of hyperparameters are fixed to the values reported in Section~\ref{sec:hyperparameters_select}.}
    \label{fig:ablation_hyperparameters}
\end{figure}

\section{Training Time Analysis}
In Table~\ref{tab:training_time} we compare the training times of ELR and ELR+ with two state-of-the-art methods, using a single Nvidia v100 GPU. ELR+ is twice as slow as ELR. DivideMix takes more than 2 times longer than ELR+ to train. Co-teaching+ is about twice as slow as ELR+.  
\begin{table}[h]
\footnotesize
\begin{center}
\begin{tabular}{c| c|c|c|c}
\toprule
Co-teaching+\cite{Yu2019HowDD} & DivideMix\cite{li2020dividemix} & ELR & ELR+\\
\midrule
4.4h & 5.4h & 1.1h & 2.3h\\
\bottomrule
\end{tabular}
\end{center}
\caption{Comparison of total training time in hours on CIFAR-10 with 40\% symmetric label noise.}
\label{tab:training_time}
\end{table}

\end{document}